\pdfoutput=1
\documentclass[11pt]{article}
\usepackage{acl}
\usepackage{times}
\usepackage{latexsym}
\usepackage[T1]{fontenc}
\usepackage[utf8]{inputenc}
\usepackage{microtype}
\usepackage{inconsolata}
\usepackage{graphicx}
\usepackage{amsmath}
\usepackage{booktabs}
\usepackage{amssymb}
\usepackage{amsthm}
\usepackage{utfsym}
\usepackage{multirow}
\usepackage{algorithm}
\usepackage{algorithmic}

\usepackage{subcaption}

\newtheorem{theorem}{Theorem}
\usepackage{appendix}

\title{PIP: Perturbation-based Iterative Pruning for Large Language Models}

\author{
 \textbf{Yi Cao\textsuperscript{1}} \hspace{1em}
 \textbf{Wei-Jie Xu\textsuperscript{2}} \hspace{1em}
 \textbf{Yucheng Shen\textsuperscript{1}} \hspace{1em}
 \\
 \textbf{Weijie Shi\textsuperscript{3}} \hspace{1em}
 \textbf{Chi-Min Chan\textsuperscript{3}} \hspace{1em}
 \textbf{Jianfeng Qu\textsuperscript{1}} \hspace{1em}
 \textbf{Jiajie Xu\textsuperscript{1}\thanks{
Corresponding author: xujj@suda.edu.cn}}
\\
 \textsuperscript{1}School of Computer Science and Technology, Soochow University \\
 \textsuperscript{2}School of Artificial Intelligence, Nanjing University \\
 \textsuperscript{3}Department of Computer Science and Engineering, Hong Kong University of Science and \\ Technology
}

\begin{document}
\maketitle
\begin{abstract}
The rapid increase in the parameter counts of Large Language Models (LLMs), which often reach into the billions or even trillions, presents significant challenges for their practical deployment, particularly in resource-constrained environments. To address this issue, we propose PIP (Perturbation-based Iterative Pruning), a novel double-view structured pruning method to optimize LLMs, which combines information from two different views: the unperturbed view and the perturbed view. With the calculation of gradient differences, PIP iteratively prunes those that struggle to distinguish between these two views. Our experiments show that PIP reduces the parameter count by approximately 20\% while retaining over 85\% of the original model's accuracy across varied benchmarks. In some cases, the performance of the pruned model is within 5\% of the unpruned version, demonstrating PIP's ability to preserve key aspects of model effectiveness. Moreover, PIP consistently outperforms existing state-of-the-art (SOTA) structured pruning methods, establishing it as a leading technique for optimizing LLMs in constrained environments.
\end{abstract}

\section{Introduction}
Large Language Models (LLMs) \cite{achiam2023gpt4,dubey2024llama} based on the Transformer architecture \cite{vaswani2017transformer} have demonstrated impressive capabilities across a wide range of tasks, but their capabilities come at the expense of massive parameter counts and high computational requirements \cite{kaplan2020scalinglaw}. For illustration, the LLaMA3-405B model \cite{dubey2024llama}, with about 405 billion parameters, demands at least 810 GB of memory with 11 A100 GPUs when using half-precision (FP16) format. Therefore, an issue presents itself that warrants further exploration: Can we produce a smaller, general-purpose, and competitive LLM by leveraging existing pre-trained LLMs, while using much less compute than training one from scratch? 
\cite{Xia2023ShearedLA} 

To this end, researchers have been exploring strategies like pruning~\cite{Frantar2023SparseGPT,ma2023llm,ashkboos2024slicegpt}, quantization~\cite{bai2021quantizationbinarybert,Ji2024quantizationAWQ}, knowledge distillation~\cite{sun2020kdcontrastive,pan2021kdmeta}, and low-rank factorization~\cite{saha2023lowrankmatrix,yuan2023lowrankasvd}. Among them, pruning methods have gained considerable attention due to their potential to significantly reduce model size while preserving performance. Yet, traditional pruning strategies typically assess importance through isolated metrics such as weight magnitudes or input-output similarity~\cite{men2024shortgpt}. While intuitive, these single-view approaches suffer from fundamental flaws: they overlook the capacity to preserve semantic robustness under adversarial or natural input variations, which is essential for ensuring reliable deployment of models in real-world applications.

To overcome the limitation, we propose a novel double-view approach that evaluates layer importance based on their awareness of text perturbations. For each input, we generate two complementary perspectives: an original sample and its perturbed counterpart—crafted via character-level edits that preserve syntax but maximally distort meaning (Section~\ref{subsection:perturbation}). By contrasting parameter gradients between these views through first-order Taylor approximation, we identify layers exhibiting weak semantic discrimination.

Beyond the initial double-view comparison, our approach employs an iterative gradient reassessment strategy to further refine the pruning process. After pruning the least sensitive layers identified in each cycle, we proceed to recompute gradient differences on the updated architecture. This dynamic process, which is akin to curriculum learning, progressively focuses on layers that are critical for semantic stability. By doing so, it ensures a thorough and comprehensive importance evaluation through successive approximations.

Our contributions are summarized as follows:
\begin{itemize}
\item
We introduce PIP, which is a structured pruning approach designed to iteratively remove low-importance layers identified by PertImport (detailed in Section \ref{subsection:pertimport}) and recomputes gradients on the pruned architecture. PIP can be seamlessly integrated into popular LLM frameworks, such as Hugging Face, with minimal code modifications, offering a lightweight yet theoretically sound implementation.

\item 
Through extensive experiments, we demonstrate PIP's consistent superiority over current state-of-the-art structured pruning methods. Ablation studies confirm that both perturbation (preserving semantic integrity) and the iteration process (dynamic importance reassessment) are essential for high-accuracy pruning. Additionally, comprehensive analyses further provide actionable insights for performance.

\end{itemize}

\begin{figure*}[tb]
    \centering
    \begin{subfigure}[b]{0.3\textwidth}
        \includegraphics[width=\textwidth]{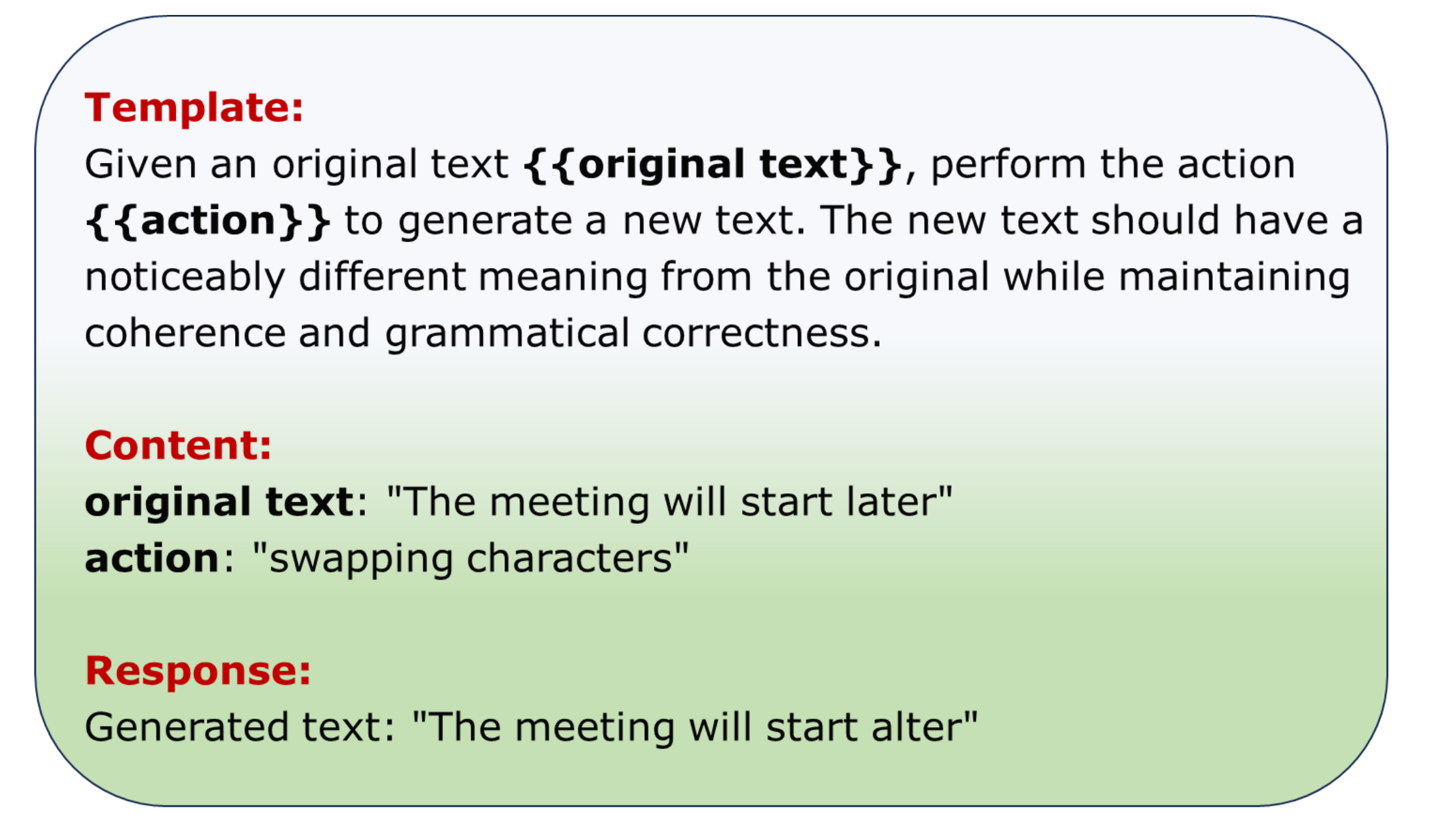}
        \caption{Character Swap}
        \label{fig:swap characters}
    \end{subfigure}
    \hfill
    \begin{subfigure}[b]{0.3\textwidth}
        \includegraphics[width=\textwidth]{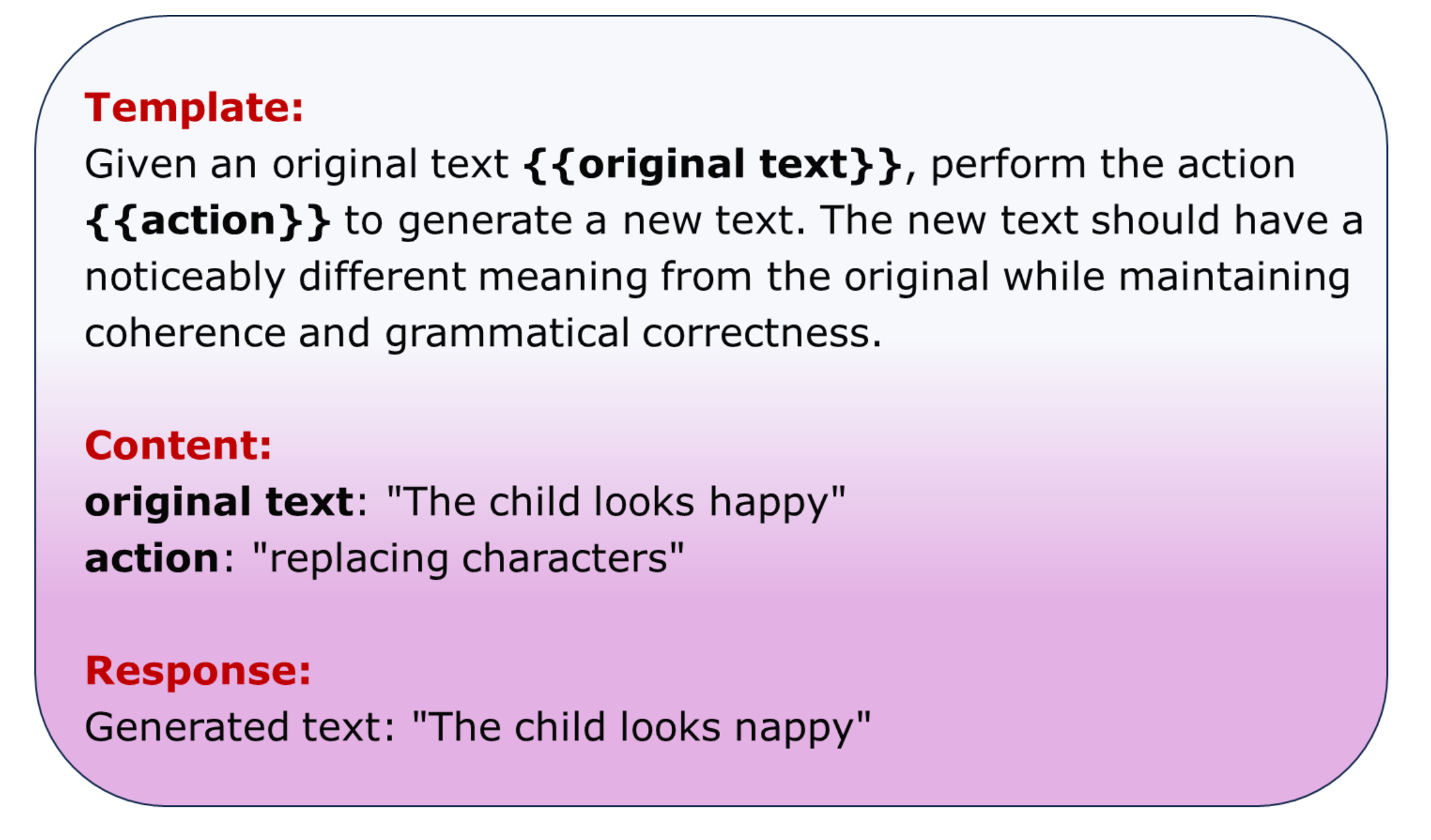}
        \caption{Character Replacement}
        \label{fig:replace characters}
    \end{subfigure}
    \hfill
    \begin{subfigure}[b]{0.3\textwidth}
        \includegraphics[width=\textwidth]{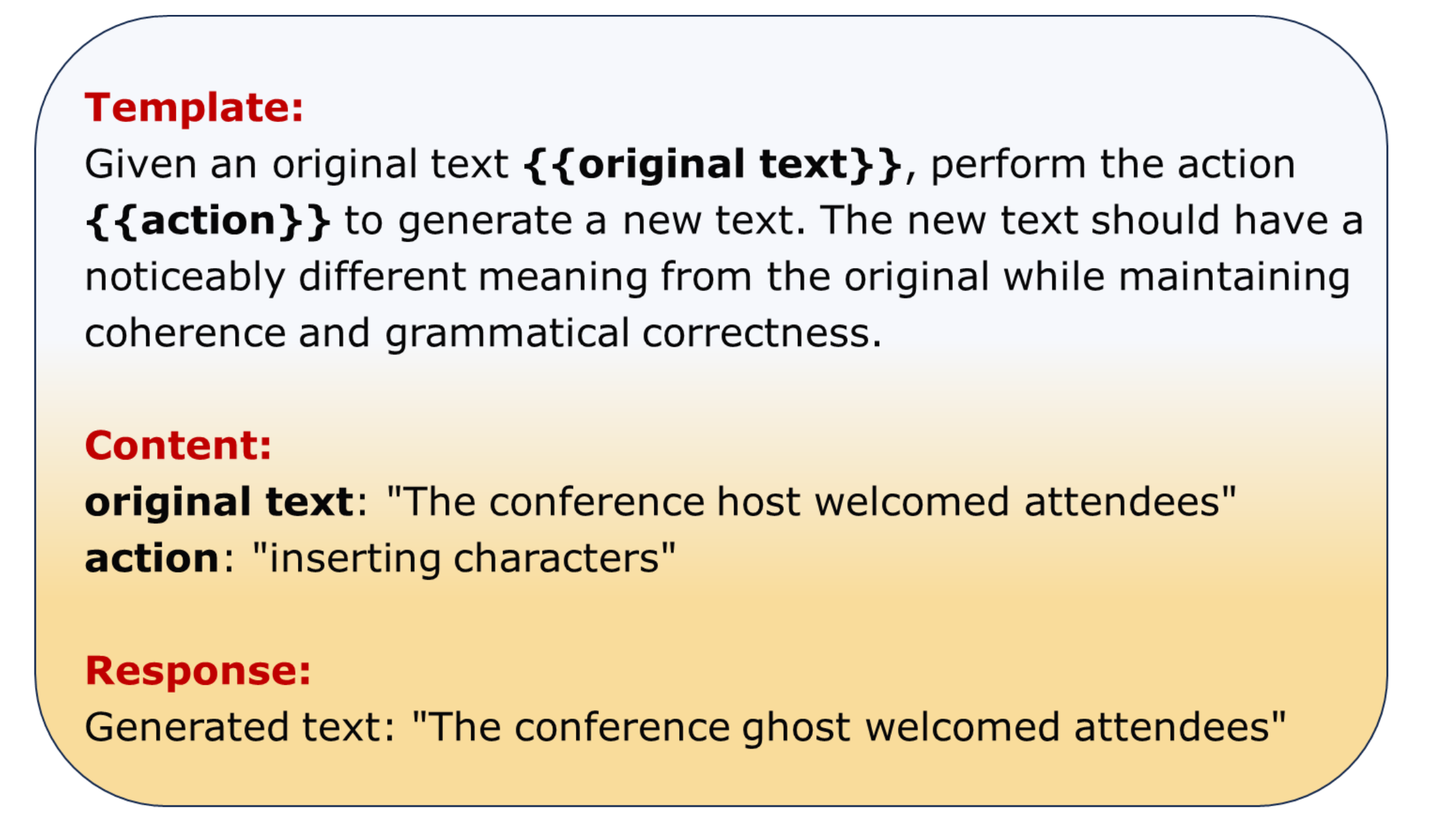}
        \caption{Character Insertion}
        \label{fig:insert characters}
    \end{subfigure}
    \caption{Generation of perturbed texts via auxiliary LLM.}
    \label{fig:generate_perturbed_samples}
\end{figure*}

\section{Related Work}
Pruning techniques for LLMs can generally be classified into two categories: unstructured pruning and structured pruning. Unstructured pruning sparsifies weight matrices by setting individual elements to zero, which often requires specialized hardware support. Notable works in this area include~\cite{Frantar2023SparseGPT,sun2023wanda}. 

Structured pruning, on the other hand, focuses on eliminating predefined units within the model, which makes it more compatible with hardware constraints. The concept of structured pruning for LLMs is introduced by \cite{wang2019structuredpruning}, which proposes parameterizing each weight matrix via low-rank factorization and actually adaptively removing rank-1 components during training. This pioneering work has truly led to the development of several other methods, such as \cite{xia2022Cofi} and \cite{Davy2024TailoredLLaMA}. These methods, however, are primarily designed for compression within specific domains or tasks, falling under the category of task-specific compression. While effective for their intended applications, they often limit the versatility of LLMs as general task solvers.

In contrast, \cite{ma2023llm} introduces a genetic pruning framework called LLM-Pruner, which aims to maintain task-agnostic capabilities while minimizing reliance on the original training dataset. Following the pipeline proposed by \cite{kwon2022pruningpipeline}, LLM-Pruner consists of three stages: Discovery, Estimation, and Recovery. It selectively removes non-critical coupled structures based on dependency analysis \cite{fang2023depgraph}, preserving the core functionality of the model. However, its potential integration with LoRA \cite{hu2021lora} presents several challenges in achieving an optimal balance between efficiency and performance.

Inspired by the approach of LLM-Pruner, the research community has proposed several methods for structured pruning in general tasks. These methods can be broadly categorized into two main types: width pruning and depth pruning \cite{kim2024shortenedllama}. Width pruning focuses on compressing the weight matrix by reducing its hidden dimension, while depth pruning targets the pruning of layers or blocks within the model. For example, ShearedLLaMA \cite{Xia2023ShearedLA} implements structured pruning through a combination of targeted pruning and dynamic batch loading. Targeted pruning removes specific layers of the model in an end-to-end fashion to achieve a predefined compression ratio. Dynamic batch loading adjusts the composition of training data batches based on the varying losses from different domains. Although this method achieves competitive performance, it suffers from the same retraining challenges as LLM-Pruner \cite{ma2023llm}. 

To avoid retraining, which can be resource-intensive and time-consuming, \cite{men2024shortgpt} proposes ShortGPT, a method based on layer importance. It introduces a novel importance metric called Block Influence, which quantifies the importance of each layer by calculating the similarity between the inputs and outputs of each layer. Layers with low importance scores are then removed. Similarly, \cite{kim2024shortenedllama} proposes Shortened LLaMA, a block-importance-based method that removes blocks based on a block-level importance metric. Another related work, SLEB \cite{song2024sleb}  evaluates the importance of Transformer blocks using the similarity between inputs and outputs, and removes the blocks with low importance scores. While these methods are straightforward to understand and implement, they fail to provide strong empirical results and lack rigorous theoretical support. Moreover, these single-view approaches are inherently limited as they neglect the necessity to maintain semantic robustness under adversarial or natural input variations, which is essential for reliable deployment.

In summary, while existing pruning methods offer trade-offs in terms of model efficiency and performance, they often either require retraining or lack solid guarantees in theory, limiting their applicability to real-world scenarios.

\section{Methodology}

\subsection{Text Perturbation}
\label{subsection:perturbation}

Text perturbation is a data augmentation technique \cite{Guerrero2023AdversarialTP} that introduces variability into textual data by applying a suite of carefully designed transformations to the original text samples.

Inspired by adversarial training \cite{Ganin2015DomainAdversarialTO}, we design text perturbations that cause radical semantic shifts while preserving grammatical correctness. Using LLM-powered prompt templates (Figure \ref{fig:generate_perturbed_samples}), we propose methods generating perturbed text samples that challenge robustness:
\subsubsection{Character Swap}
\begin{itemize}
    \item \textbf{Example:} Swapping the characters ``l'' and ``a'' in \textit{``later''} gives us \textit{``alter''}. Consequently, the sentence \textit{``The meeting will start later''} becomes \textit{``The meeting will start alter''}.
    \item \textbf{Impact:} In scheduling systems, this perturbation causes rescheduling forms to be generated instead of acknowledging delays, disrupting calendar management. In Q\&A systems, models respond with operational directives like ``How to adjust the meeting?'' instead of factual answers, spreading  incorrect procedural guidance and increasing inefficiencies.
\end{itemize}
\subsubsection{Character Replacement}
\begin{itemize}
    \item \textbf{Example:} Replacing the character ``h'' with ``n'' in \textit{``happy''} results in \textit{``nappy''}. As a result, the sentence \textit{``The child looks happy''} becomes \textit{``The child looks nappy''}.
    \item \textbf{Impact:} In dialogue systems, this perturbation leads to inappropriate suggestions like ``Check diaper supplies'' instead of emotional support, causing nonsensical interactions in childcare applications. In healthcare chatbots, it can misinterpret ``The patient feels nappy'' as a clinical symptom, leading to incorrect medical advice and eroding trust in systems.
\end{itemize}

\begin{figure*}[tb]
\hspace{-1.5em}
\includegraphics[width=1.05\textwidth,height=0.6\textwidth]{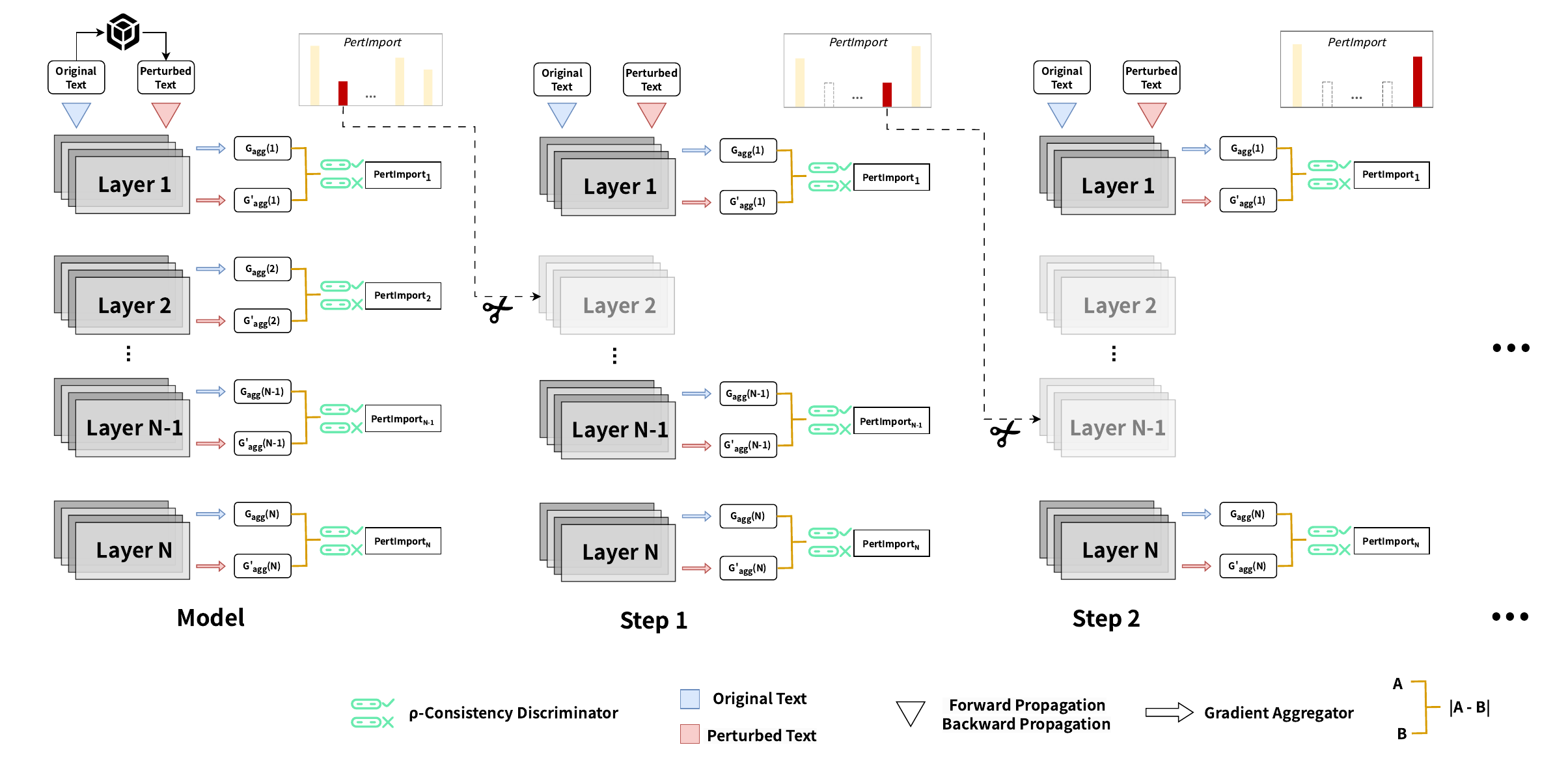}
\caption{Overview of our PIP method, where an auxiliary LLM generates perturbed text (see Section \ref{subsection:perturbation}). During pruning, the same original and perturbed texts are used to compute PertImport values, determining layer significance. The least significant layers (e.g., the 2nd, ($N-1$)-th, and $N$-th layers in the first three steps) are iteratively removed.}
\label{fig:main}
\end{figure*}

\subsubsection{Character Insertion}
\begin{itemize}
    \item \textbf{Example:} Inserting the character ``g'' to \textit{``host''} yields \textit{``ghost''}. Then, the sentence \textit{``The conference host welcomed attendees''} becomes \textit{``The conference ghost welcomed attendees''}.
    \item \textbf{Impact:} In automated summarization, this perturbation can generate fictional narratives, misrepresenting factual events. In enterprise search, it may retrieve irrelevant documents, introducing noise into enterprise knowledge graphs and decision-making pipelines.
\end{itemize}

\subsection{PertImport: A Perturbation-based Metric for Layer Importance Assessment}\label{subsection:pertimport}

Building on the text perturbation framework defined in Section~\ref{subsection:perturbation}, we propose PertImport, a novel metric to measure the sensitivity to meaning-altering inputs. The rationale for this metric is grounded in the following analysis:

Consider a pre-trained large language model $\mathcal M$ with $N$ layers. Each layer $i$ has parameters \( \mathbf{w}_i \in \mathbb{R}^n \). Excluding embedding and head layers, $\mathcal M$ can be seen as a mapping function $f$. For any input $s$, the function $f(s; \mathbf{w}_1, \mathbf{w}_2, \ldots, \mathbf{w}_{N})$ generates an output that is consistent with $s$'s semantics.

In Supervised Fine-Tuning (SFT) \cite{brown2020gpt3}, when sample \( s \) is used as both input and label for model \( \mathcal{M} \), the parameter update rule is: 
\begin{align}\label{eq:weight_update}
\mathbf{w}_i' = \mathbf{w}_i - \alpha \cdot \nabla_{\mathbf{w}_i} \mathcal{L} (s; \mathbf{w}_{1:N})
.\end{align}
Introducing perturbation $\delta s$ (Section \ref{subsection:perturbation}) to $s$ yields the perturbed sample $s + \delta s$ for SFT:
\begin{align}\label{eq:weight_update_ps}
\mathbf{w}_i'' = \mathbf{w}_i - \alpha \cdot \nabla_{\mathbf{w}_i} \mathcal{L} (s + \delta s; \mathbf{w}_{1:N}).
\end{align}
Here, \( \mathbf{w}_i \in \mathbb{R}^n \) denotes the original parameters of the \( i \)-th layer, \( \mathbf{w}_i^{'} \in \mathbb{R}^n \) and \( \mathbf{w}_i^{''} \in \mathbb{R}^n \) denote the updated parameters of the \( i \)-th layer. The notation $\mathbf{w}_{1:N}$ represents the collection of parameters from the first layer to the $N$-th layer. The learning rate is denoted by \( \alpha \) , and the loss function \( \mathcal{L}(\cdot) \) quantifies the difference between predictions and labels.

Given that $f$ is differentiable and $\delta s \to \mathbf{0}$, we can use the first-order Taylor expansion to approximate the change in the function value. Specifically, we obtain the following approximation:
\begin{align}\label{eq:taylor}
& f(s + \delta s; \mathbf{w}_{1:N}^{''})  - f(s; \mathbf{w}_{1:N}^{'}) \nonumber \\
 & \approx \nabla_{s} f \cdot \delta s + \sum_{i,j} (\nabla_{\mathbf{w}^{'}_i} f)_{j} \cdot (\mathbf{w}_i^{''} - \mathbf{w}_i^{'})_{j}
,\end{align}
where $(\nabla_{\mathbf{w}^{'}_i} f)_{j} \in \mathbb{R}$ represents the $j$-th element of the gradient vector $\nabla_{\mathbf{w}^{'}_i} f$ at the $i$-th layer, and $(\mathbf{w}_i^{''} - \mathbf{w}_i^{'})_{j} \in \mathbb{R}$ represents the $j$-th element of the parameter difference $\mathbf{w}_i^{''} - \mathbf{w}_i^{'}$ at the $i$-th layer. 

Subsequently, we utilize Equation (\ref{eq:taylor}) to establish an upper bound for the estimation of the difference in output values with and without the perturbation $\delta s$. We introduce a constant sequence $\left \{ C_i \right \} _{i=1}^{N} $ defined as $C_i=\underset{1 \le j \le n}{max} \left | (\nabla_{\mathbf{w}^{'}_i} f)_{j} \right | $. 
Incorporating Equations (\ref{eq:weight_update}), (\ref{eq:weight_update_ps}), and (\ref{eq:taylor}), we arrive at:
\begin{align}\label{eq:output_estimate}
& \left | f(s + \delta s; \mathbf{w}_{1:N}^{''})  - f(s; \mathbf{w}_{1:N}^{'}) \nonumber \right | \nonumber \\
& \le \left | \nabla_{s} f \cdot \delta s \right | + \nonumber \\
&  C \sum_{i,j} \left | (\nabla_{\mathbf{w}_i} (\mathcal{L}(s + \delta s;\mathcal W) - \mathcal{L}(s;\mathcal W)))_j \right | 
,\end{align}
where $C = \underset{1 \le i \le N}{max}\left \{ C_i \right \}\in \mathbb{R}$ is the maximum value, and $\mathcal W$ represents all the parameters of $\mathcal M$.

\begin{theorem} \label{thm}
To enhance the robustness of the pruned model (defined as its capability to distinguish between $s$ and $s+\delta s$), it's best to select parameters with smaller gradient differences between the perturbed and unperturbed views.
\end{theorem}

\begin{proof}
Let $\mathbf{Y}$ be a random variable representing the output difference with and without the perturbation $\delta s$. Consider removing the parameter at the $i_0$-th layer and the $j_0$-th position, i.e., $(\mathbf{w}_{i_0})_{j_0}$. Suppose there exists another parameter $(\mathbf{w}_{i_0^{'}})_{j_0^{'}}$ that has a smaller gradient difference and a higher average probability of detecting the difference between $s$ and $s+\delta s$. Assuming, without loss of generality, that $\mathbf{Y}$ follows a uniform distribution (as an analytical tool based
on the principle of maximum entropy), we can derive the expectation of $\mathbf{Y}_{\overline{{i_0}, {j_0}}}$ after pruning the parameter $(\mathbf{w}_{i_0})_{j_0}$:
\begin{align}\label{eq:E_X_i0_j0}
& \sum_{(i, j)\ne (i_0,j_0)}^{} \left | (\nabla_{\mathbf{w}_i} (\mathcal{L}(s + \delta s;\mathcal W) - \mathcal{L}(s;\mathcal W)))_j \right | \nonumber \\
& = \frac{2}{C} \mathbb{E}[\mathbf{Y}_{\overline{{i_0}, {j_0}}}] - \frac{1}{C} \left | \nabla_{s} f \cdot \delta s \right |
.
\end{align}
Similarly, when $(\mathbf{w}_{i_0^{'}})_{j_0^{'}}$ is pruned, we can derive the equation for the expectation of $\mathbf{Y}_{\overline{{i_0^{'}}, {j_0^{'}}}}$:
\begin{align}\label{eq:E_X_i00_j00}
& \sum_{(i, j)\ne (i_0^{'},j_0^{'})}^{} \left | (\nabla_{\mathbf{w}_i} (\mathcal{L}(s + \delta s;\mathcal W) - \mathcal{L}(s;\mathcal W)))_j \right | \nonumber \\
& = \frac{2}{C} \mathbb{E}[\mathbf{Y}_{\overline{{i_0^{'}}, {j_0^{'}}}}] - \frac{1}{C} \left | \nabla_{s} f \cdot \delta s \right |
.\end{align}
By Equations (\ref{eq:E_X_i0_j0}) and (\ref{eq:E_X_i00_j00}), we have $\mathbb{E}[\mathbf{Y}_{\overline{{i_0}, {j_0}}}]>\mathbb{E}[\mathbf{Y}_{\overline{{i_0^{'}}, {j_0^{'}}}}]$. This contradicts the hypothesis that removing $(\mathbf{w}_{i_0^{'}})_{j_0^{'}}$ increases the likelihood of detecting the perturbation. Consequently, the assumption is invalid, implying that the theorem holds.
\end{proof}

Based on Theorem \ref{thm}, we propose a robustness-aware importance metric, PertImport. For the $i$-th layer, PertImport quantifies its discriminative sensitivity through the following value:
\begin{align}
 \frac{1}{|\mathcal D|} \sum_{s\in \mathcal D}^{} & |g(\nabla_{\mathbf{w}_i} \mathcal{L}(s + \delta s;\mathcal W)_{1:n}) \nonumber \\ 
- & g(\nabla_{\mathbf{w}_i} \mathcal{L}(s;\mathcal W)_{1:n}) |
,\end{align}
where $\mathcal D$ is a small calibration dataset, and $n$ is the count of parameters in the $i$-th layer of $\mathcal M$. The function $ g: \mathbb{R}^n \to \mathbb{R}$ aggregates gradient information for a specific layer using norms like $L_1$, $L_2$, or $L_\infty$, as shown in \cite{han2015deepcompression}. See Appendix \ref{appendix:norm} for the definitions of these norms.

\subsection{PIP: Perturbation-based Iterative Pruning}

After assessing layer importance via perturbation, we avoid making premature pruning decisions. Instead, PIP uses a more systematic approach by employing an iterative greedy strategy to progressively prune layers with minimal performance impact. 

To enhance the robustness and accuracy of the importance evaluation, we introduce a consistency discriminator that filters out layers with unstable gradient differences. Specifically, it computes the standard deviation across multiple perturbations and excludes layers with high variability.

In summary, by integrating information from both the unperturbed and perturbed views, PIP effectively reduces the stochasticity inherent in pruning and facilitates a more stable and efficient optimization. The details are presented in Algorithm \ref{alg:PIP}, and the overall workflow is illustrated in Figure \ref{fig:main}.

\begin{algorithm}[tb]
    \caption{Detailed Implementation of PIP}
    \label{alg:PIP}
    \begin{algorithmic}[0]
        \setlength{\baselineskip}{1.30em}

        \REQUIRE \qquad \\ pre-trained LLM $\mathcal M$, \\
        calibration dataset $\mathcal D$, \\
        text perturbation method $\delta(\cdot)$, \\
        gradient aggregation function $g(\cdot)$, \\
        \# of layers to be pruned $L \in \mathbb{N} ^{*}$, \\
        consistency threshold $\rho \in \mathbb{R}^{+}$
        \ENSURE \qquad \\ indices of the pruned layers $\mathcal P$, \\
        pruned model $\mathcal M^{*}$

        \STATE \hspace{-4mm} $\mathcal D^{'}\gets \delta(\mathcal D);$

        \STATE \hspace{-4mm} $\mathcal M^{*}\gets \mathcal M;$
        
        \STATE \hspace{-4mm} $N_{\mathcal M}\gets |\mathcal M.layers|;$

        \STATE \hspace{-4mm} $\mathcal P\gets \emptyset;$

        \STATE \hspace{-4mm} \textbf{for } $\ell = 1$ to $L$ \textbf{do}
        \STATE  
        $ \left \{ \boldsymbol{G}_{agg}(i) \right \} _{i=1}^{N_{\mathcal M}-\ell+1} \gets(g\circ SFT)(\mathcal M^{*}, \mathcal D);$ 
        
        \STATE 
        $\left \{ \boldsymbol{G'}_{agg}(i) \right \} _{i=1}^{N_{\mathcal M}-\ell+1} \gets(g\circ SFT)(\mathcal M^{*}, \mathcal D^{'});$ 

        \STATE
        \textbf{for} $i = 1$ to $N_{\mathcal M}-\ell+1$ \textbf{do}

        \STATE
        \hspace{4mm} \textbf{if} $std(\boldsymbol{G}_{agg}(i)-\boldsymbol{G'}_{agg}(i))<\rho$ \textbf{then} 

        \STATE
        \hspace{8mm} $ \boldsymbol {PI}(i)  \gets |\boldsymbol{G}_{agg}(i)-\boldsymbol{G'}_{agg}(i)|;$

        \STATE
        \hspace{4mm} \textbf{else}

        \STATE
        \hspace{8mm} $ \boldsymbol {PI}(i)  \gets +  \infty ;$

        \STATE
        \hspace{4mm} \textbf{end if}

        \STATE
        \textbf{end for}

        \STATE 
        $p_{\ell}\gets {argmin}_{1\le i\le N_{\mathcal M}-\ell+1}^{}\boldsymbol {PI}(i);$
        
        \STATE 
        $\mathcal P\gets \mathcal P \cup \left \{ p_{\ell} \right \};$ 
        
        \STATE 
        $\mathcal M^{*} \gets \mathcal M^{*} \setminus \left \{ layer_{p_{\ell}} \right \};$ 
        \STATE \hspace{-4mm} \textbf{end for}
        \STATE \hspace{-4mm} \textbf{return} $\mathcal P$, $\mathcal M^{*}$
    \end{algorithmic}
\end{algorithm}

\section{Experiments}

\subsection{Experimental Setup}

\subsubsection{Model Selection}
To compare with existing methods, we conduct experiments on LLaMA2~\cite{touvron2023llama} and LLaMA3~\cite{dubey2024llama} models of varying sizes. The architectural similarity to other LLMs allows our method, PIP, to generalize effectively to other models. Experiments on additional model architectures are provided in Appendix \ref{appendix:other_llm}.
\subsubsection{Evaluation and Datasets}
We evaluate accuracy using the following datasets: BoolQ \cite{clark2019boolq}, PIQA \cite{bisk2020piqa},  HellaSwag \cite{zellers2019hellaswag}, WinoGrande \cite{sakaguchi2021winogrande}, ARC-Easy \cite{clark2018arc}, 
ARC-Challenge \cite{clark2018arc}, and OpenBookQA \cite{OpenBookQA2018}, all of which have been widely utilized in previous structured pruning studies. To ensure a fair comparison across the aforementioned datasets, we use the LM Evaluation Harness framework \cite{eval-harness} with its default settings for evaluation, without incorporating any shots as demonstrations. In addition, to assess the capability of predicting the next token, we evaluate perplexity (PPL) on the PTB dataset \cite{marcus1993ptb}, which is an established metric for evaluating the predictive capabilities of LLMs. 
\subsubsection{Baseline Methods}
To show the effectiveness of our PIP method, we compare it with several state-of-the-art structured pruning methods specifically designed for LLMs: 
\begin{itemize}
\item LLM-Pruner \cite{ma2023llm}: A method which uses Taylor-based metrics to prune less important heads in MHA and neurons in FFN.

\item SliceGPT \cite{ashkboos2024slicegpt}: A method which applies orthogonal transformations. By doing this, it can prune both rows and columns of the weight matrices, which in turn helps to reduce the hidden size within the LLM.

\item ShortGPT \cite{men2024shortgpt}: A method which identifies redundant layers that have a small similarity between the inputs and outputs, pruning those to reduce the depth.
\end{itemize}

To compare with baselines, we follow the same experimental settings suggested in their studies.

\subsubsection{Experimental Details}
Following \cite{ashkboos2024slicegpt}, we randomly select a few samples (fewer than 10) from the WikiText2 dataset \cite{merity2016wikitext2} for calibration, ensuring reproducibility with a fixed random seed. We aggregate gradients using the $L_2$-norm. Experiments are conducted using the Transformers library \cite{transformers} on a server with 8 NVIDIA A100 GPUs (80GB VRAM each, totaling 640GB).

\subsubsection{Statistics of Pruned Models}
Table~\ref{tab:pruned-stats} summarises the key characteristics of the pruned models used in our primary experiments, including parameter count, memory footprint, and Time-Per-Output-Token (TPOT). Evaluations use a randomly sampled sequence from WikiText2 with a fixed output length of 128 tokens. For hardware configuration: LLaMA2-8B and LLaMA3-13B are tested on a single NVIDIA A100 GPU, while LLaMA2-70B and LLaMA3-70B employ tensor parallelism across 4 NVIDIA A100 GPUs. All experiments are executed in half-precision mode.
\begin{table}[tb]
    \centering
    \setlength{\tabcolsep}{0.4mm}{
    \begin{tabular}{c|c|rrrrr}
    \toprule 
    \toprule 
    \textbf{Model} & \textbf{Ratio} & \textbf{\#Params} & \textbf{Memory} & \textbf{TPOT} \\ 
    \midrule
    \multirow{2}{*}{LLaMA3-8B}  & \textit{Dense} & 8.0B & 15.0GiB & 46.7ms \\
                                & 19.0\% & 6.5B & 12.1GiB & 41.5ms \\ 
    \midrule
    \multirow{2}{*}{LLaMA3-70B} & \textit{Dense} & 70.6B & 131.4GiB & 266.2ms \\
                                & 19.4\% & 56.9B & 99.5GiB & 223.5ms \\ 
    \midrule
    \multirow{2}{*}{LLaMA2-13B} & \textit{Dense} & 13.0B & 24.4GiB & 73.4ms \\
                                & 19.5\% & 10.5B & 19.6GiB & 58.9ms \\ 
    \midrule
    \multirow{2}{*}{LLaMA2-70B} & \textit{Dense} & 69.0B & 128.5GiB & 269.7ms \\
                                & 19.9\% & 55.3B & 96.6GiB & 217.8ms \\ 
    \bottomrule
    \bottomrule
    \end{tabular}
    }
    \caption{Statistics of base and pruned models. \textit{``Dense''} denotes the base model. ``Ratio'' is the pruning ratio, calculated as $(\text{\#Pruned Params})/({\text{\#Base Params}})$.}
    \label{tab:pruned-stats}
\end{table}

\subsection{Zero-shot Performance}

\begin{table*}[tb]
    \setlength{\tabcolsep}{0.4mm}{
    \begin{tabular}{ll|r|rrrrrrr|r}
    \toprule \toprule 
    \multicolumn{1}{c}{\textbf{Model}}       & \textbf{Method}      & \textbf{PPL↓} & \textbf{BoolQ↑} & \textbf{PIQA↑} & \textbf{HeSwg↑} & \textbf{WGrd↑} & \textbf{ARC-E↑} & \textbf{ARC-C↑} & \textbf{OBQA↑} & \textbf{Avg.↑} \\ 
    \midrule
    LLaMA3-8B   & \textit{Dense}        & 10.6 & 81.4 & 79.7 & 60.2 & 72.5 & 80.1 & 50.5 & 34.8 & 65.6 \\ \cline{2-11}
                & LLM-Pruner   & \underline{56.5} & 63.5 & \underline{69.5} & 42.6 & 62.4 & 52.3 & 29.5 & \textbf{27.0} & 49.5 \\
                & SliceGPT     & 72.3 & 40.8 & 65.4 & 39.8 & 63.2 & \textbf{59.5} & 29.4 & 23.8 & 46.0 \\
                & ShortGPT     & 67.9 & \underline{65.2} & 68.9 & \textbf{45.6} & \textbf{69.4} & 57.2 & \textbf{36.5} & 25.4 & \underline{52.6} \\
                & PIP (Ours)  & \textbf{56.3} & \textbf{70.9} & \textbf{69.6} & \underline{44.7} & \textbf{69.4} & \underline{57.9} & \underline{35.1} & \underline{26.8} & \textbf{53.5} \\ 
    \midrule
    LLaMA3-70B  & \textit{Dense}        & 8.2 & 85.2 & 82.4 & 66.4 & 80.3 & 86.9 & 60.3 & 38.2 & 71.4 \\ \cline{2-11}
                & LLM-Pruner   &-&-&-&-&-&-&-&-&-\\
                & SliceGPT     & 78.2 & 59.4 & 69.6 & 44.4 & \underline{72.1} & 69.7 & 41.1 & \textbf{30.2} & 55.2 \\
                & ShortGPT     & \underline{13.9} & \textbf{80.9} & \underline{76.0} & \textbf{57.1} & 60.5 & \textbf{77.4} & \textbf{47.9} & 19.0 & \underline{59.8} \\
                & PIP (Ours)  & \textbf{12.5} & \underline{79.0} & \textbf{78.6} & \underline{56.1} & \textbf{73.3} & \underline{75.7} & \underline{45.3} & \underline{29.6} & \textbf{62.5} \\ 
    \midrule
    LLaMA2-13B  & \textit{Dense}       & 28.9 & 80.6 & 79.1 & 60.0 & 72.4 & 79.4 & 48.5 & 35.2 & 65.0 \\ \cline{2-11}
                & LLM-Pruner  & 150.2 & \underline{57.7} & 60.3 & 31.9 & 53.9 & 37.4 & 22.9 & 15.8 & 40.1 \\
                & SliceGPT    & 64.3 & 38.2 & \underline{65.0} & \underline{39.5} & \textbf{65.5} & \textbf{61.3} & \underline{33.4} & \textbf{28.0} & \underline{47.3} \\
                & ShortGPT    & \underline{44.6} & 49.8 & 55.5 & 39.3 & 57.1 & 49.3 & 29.9 & \underline{25.4} & 43.8 \\
                & PIP (Ours) & \textbf{41.8} & \textbf{63.3} & \textbf{74.5} & \textbf{50.5} & \underline{62.0} & \underline{58.8} & \textbf{37.4} & 25.0 & \textbf{53.1} \\ 
    \midrule
    LLaMA2-70B  & \textit{Dense}        & 14.4 & 76.6 & 81.1 & 64.0 & 77.0 & 77.8 & 51.2 & 34.8 & 66.1 \\ \cline{2-11}
                & LLM-Pruner   &-&-&-&-&-&-&-&-&- \\
                & SliceGPT     & 33.9 & 70.1 & \underline{76.3} & 52.7 & \textbf{76.6} & \textbf{76.4} & \textbf{47.0} & \textbf{32.6} & \underline{61.7} \\
                & ShortGPT     & \underline{18.5}& \underline{73.5} & 73.9 & \underline{56.0} & 72.5 & 66.7 & 39.2 & 26.8 & 58.4 \\
                & PIP (Ours)  & \textbf{17.2} & \textbf{80.7} & \textbf{77.0} & \textbf{57.8} & \underline{73.4} & \underline{70.7} & \underline{43.6} & \underline{29.4} & \textbf{61.8} \\ 
    \bottomrule \bottomrule
    \end{tabular}
    }
    \caption{Zero-shot performance of LLMs with approximately 20\% pruning ratio. ``\textit{Dense}'' is the original unpruned model. ``↑'' means higher is better and ``↓'' means lower is better. Bold values denote the best performance among pruned models, underlined values the second-best. Abbreviations: HellaSwag (HeSwg), WinoGrande (WGrd), ARC-Easy (ARC-E), ARC-Challenge (ARC-C), OpenBookQA (OBQA). ``Avg.'' represents the average score across the seven benchmarks. ``-'' indicates incompatibility between pruning methods and specific LLMs.}
    \label{tab:zero-shot}
\end{table*}

We compare PIP with baselines on zero-shot performance (see Table \ref{tab:zero-shot}). On average, PIP retains over 85\% of the base model's accuracy across benchmarks, with an approximate pruning ratio of 20\%. In some cases, its performance is within 5\% of the base model's, showing its ability to preserve the crucial aspects. PIP also consistently outperforms baselines, making it a superior LLM pruning technique. For more details, see Appendix \ref{appendix:main}.

\subsection{Ablation Analysis}

\begin{table}[tb]
\centering
\setlength{\tabcolsep}{1.7mm}{
\begin{tabular}{c|cc|r}
\toprule \toprule 
\textbf{Ratio} & \textbf{Perturbation} & \textbf{Iteration}   & \textbf{PPL↓}   \\ 
\midrule
9.8\%  & \usym{2714}       & \usym{2714}  &  30.7    \\
 &\usym{2714}       & \usym{2718}   &   31.3   \\
 &\usym{2718}       & \usym{2714}   &   43.4   \\
\midrule
19.5\% & \usym{2714}       & \usym{2714}   &  41.8   \\
 &\usym{2714}       & \usym{2718}   &  42.0   \\
 &\usym{2718}       & \usym{2714}   &  81.9  \\
\midrule
29.2\% & \usym{2714}       & \usym{2714}  &  53.3  \\
& \usym{2714}       & \usym{2718}   &  67.4  \\
& \usym{2718}       & \usym{2714}   &  99.8   \\
\bottomrule \bottomrule
\end{tabular}
}
\caption{Ablation studies of LLaMA2-13B. ``Ratio'': Pruning Ratio; ``\protect\usym{2714}'': Enabled; ``\protect\usym{2718}'': Disabled.}
\label{tab:ablation}
\end{table}

To investigate the critical components of PIP, we conduct ablation studies on LLaMA2-13B, focusing on perturbation and greedy-search iteration. As shown in Table \ref{tab:ablation}, we evaluate three pruning ratios (9.8\%, 19.5\%, 29.2\%) under three configurations: (1) Full PIP implementation, (2) Perturbation-only, and (3) Iteration-only. The performance is measured using the PPL metric on the PTB dataset.

Our results show that perturbation and iteration work synergistically. At a 29.2\% pruning ratio, the combined approach achieves optimal performance, while disabling iteration or perturbation degrades performance, with perturbation's absence causing more severe deterioration. In addition, similar trends are observed at lower pruning ratios.

\subsection{More Analysis}
\subsubsection{Effect of Gradient Aggregation}

\begin{figure}[tb]
    \centering
    \begin{subfigure}[t]{0.23\textwidth}
        \centering
        \includegraphics[width=\linewidth]{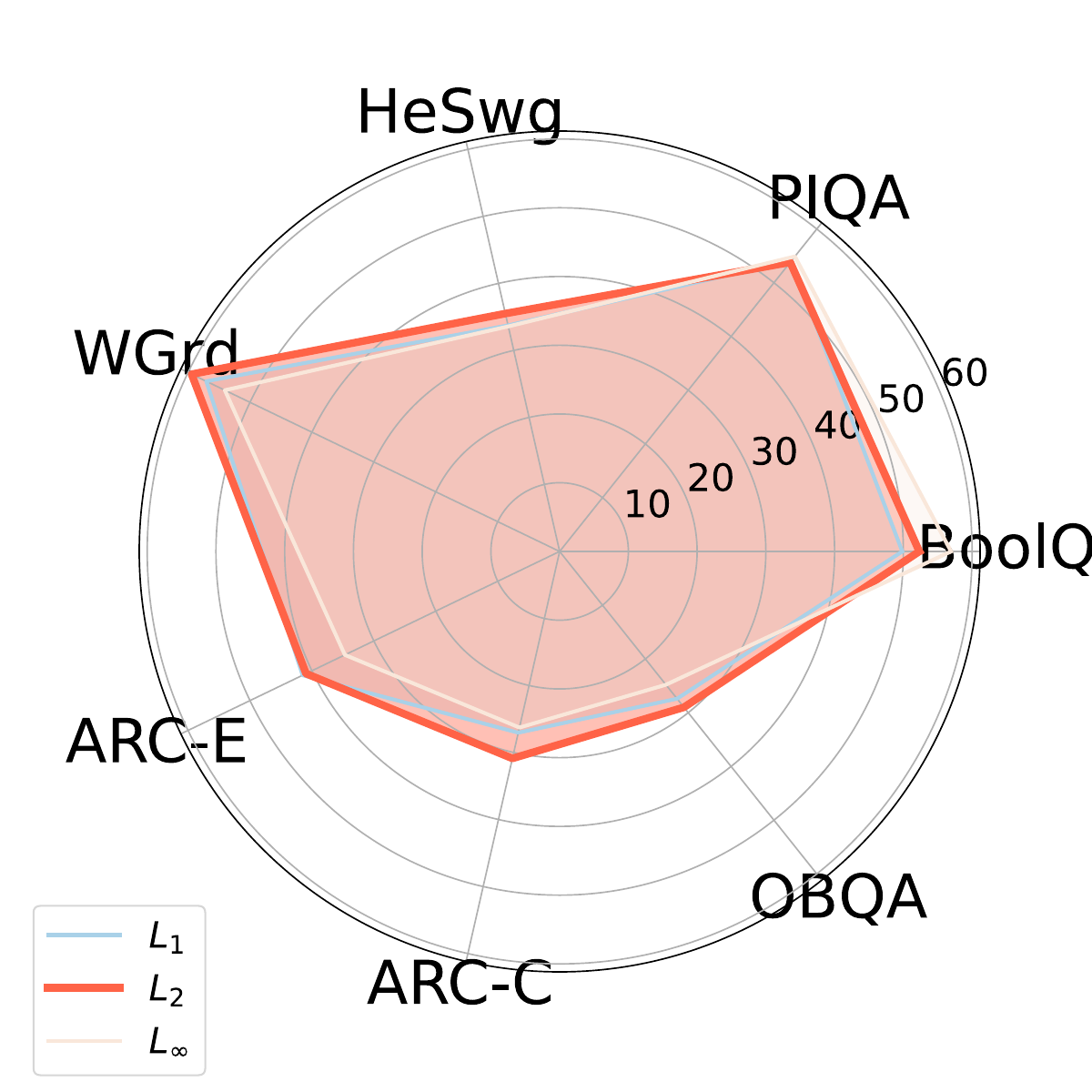}
        \caption{LLaMA3}
    \end{subfigure}
    \hfill
    \begin{subfigure}[t]{0.23\textwidth}
        \centering
        \includegraphics[width=\linewidth]{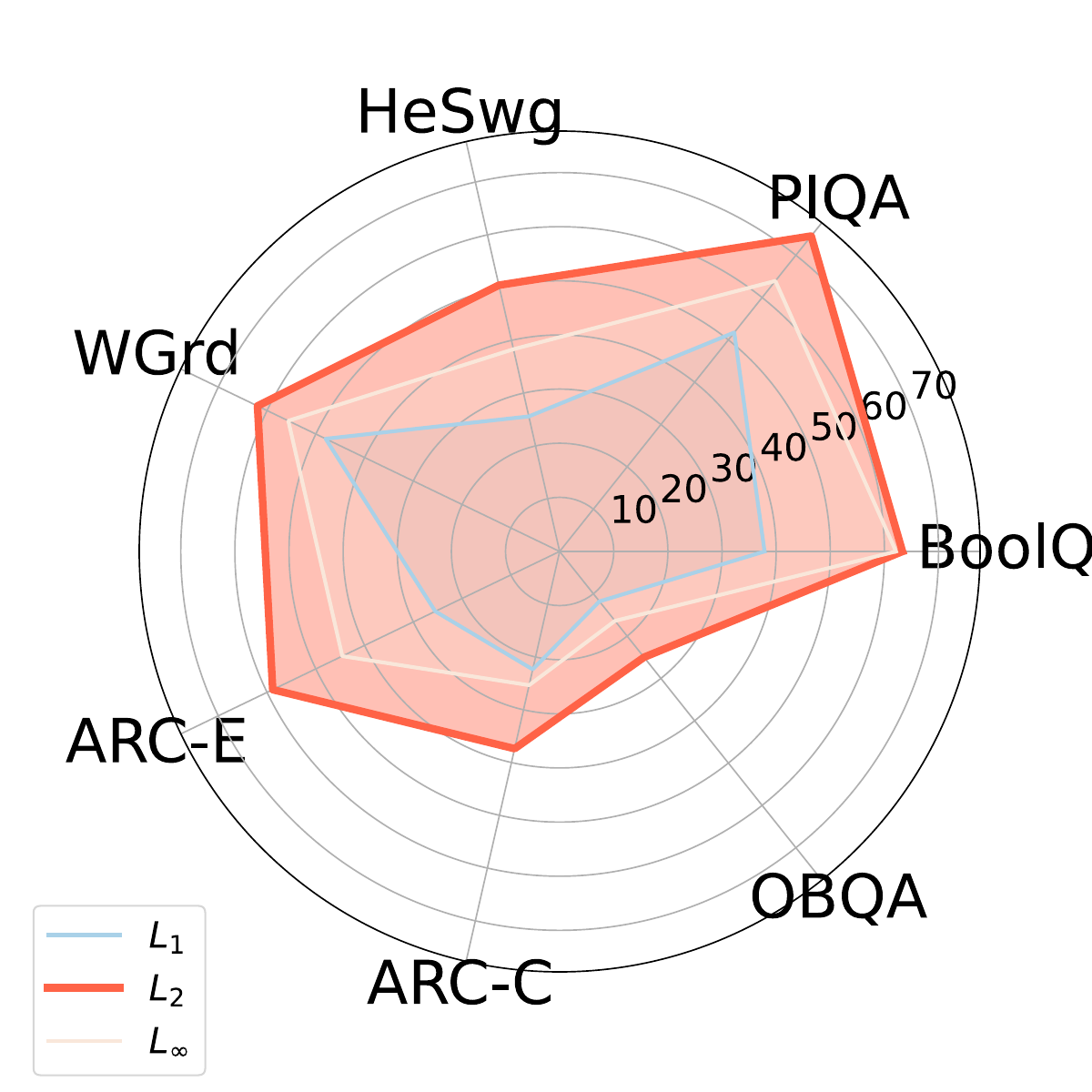}
        \caption{LLaMA2}
    \end{subfigure}
    \caption{Zero-shot performance of the pruned LLM using PIP, across various gradient aggregation strategies.}
    \label{aggregation_strategy}
\end{figure}

In this section, we explore different gradient aggregation strategies for model pruning, as shown in Figure \ref{aggregation_strategy}. Overall, the $L_2$-norm outperforms the $L_1$ and $L_{\infty }$ norms by assigning more weight to larger gradients and mitigating extreme values. Therefore, we recommend the $L_2$-norm as the default for general use. For more details, see Appendix \ref{appendix:grad_aggr}.

\subsubsection{Effect of Calibration Data Volume}
\begin{figure}[tb]
    \centering
    \begin{subfigure}[t]{0.23\textwidth}
        \centering
        \includegraphics[width=\linewidth]{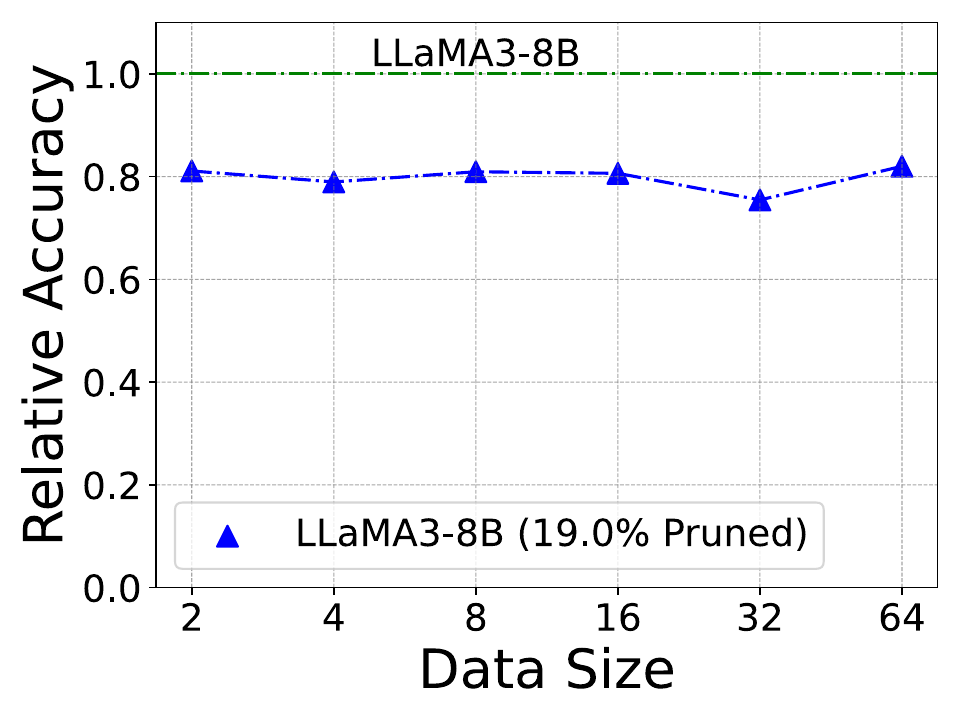}
    \end{subfigure}
    \hfill
    \begin{subfigure}[t]{0.23\textwidth}
        \centering
        \includegraphics[width=\linewidth]{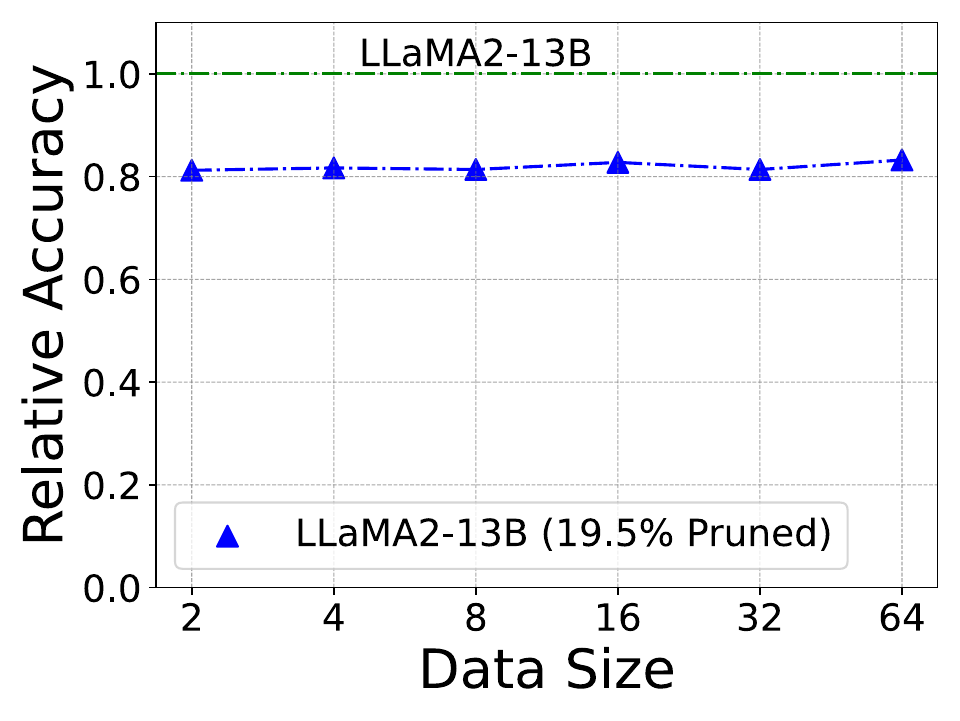}
    \end{subfigure}
    \caption{Experiments of performance at various data sizes. ``Data Size'' is the number of samples in the calibration dataset. ``Relative Accuracy'' is the ratio of the average accuracy of the pruned LLM on various benchmarks to the average accuracy of the Dense model.}
    \label{fig:data_volume}
\end{figure}
This section examines how calibration data volume affects the zero-shot performance of LLMs pruned using PIP. We use character swap and replacement techniques with the $L_2$-norm for gradient aggregation. Figure \ref{fig:data_volume} shows that PIP achieves high efficiency with minimal calibration samples, demonstrating ``few-shot'' learning capabilities. This suggests that PIP can deliver strong performance without requiring large datasets, addressing a common limitation in practical applications. For more details, readers are referred to Appendix \ref{appendix:data_volume}.

\subsubsection{Pruning Time Analysis}

\begin{figure}[tb]
    \centering
    \includegraphics[width=0.46\linewidth]{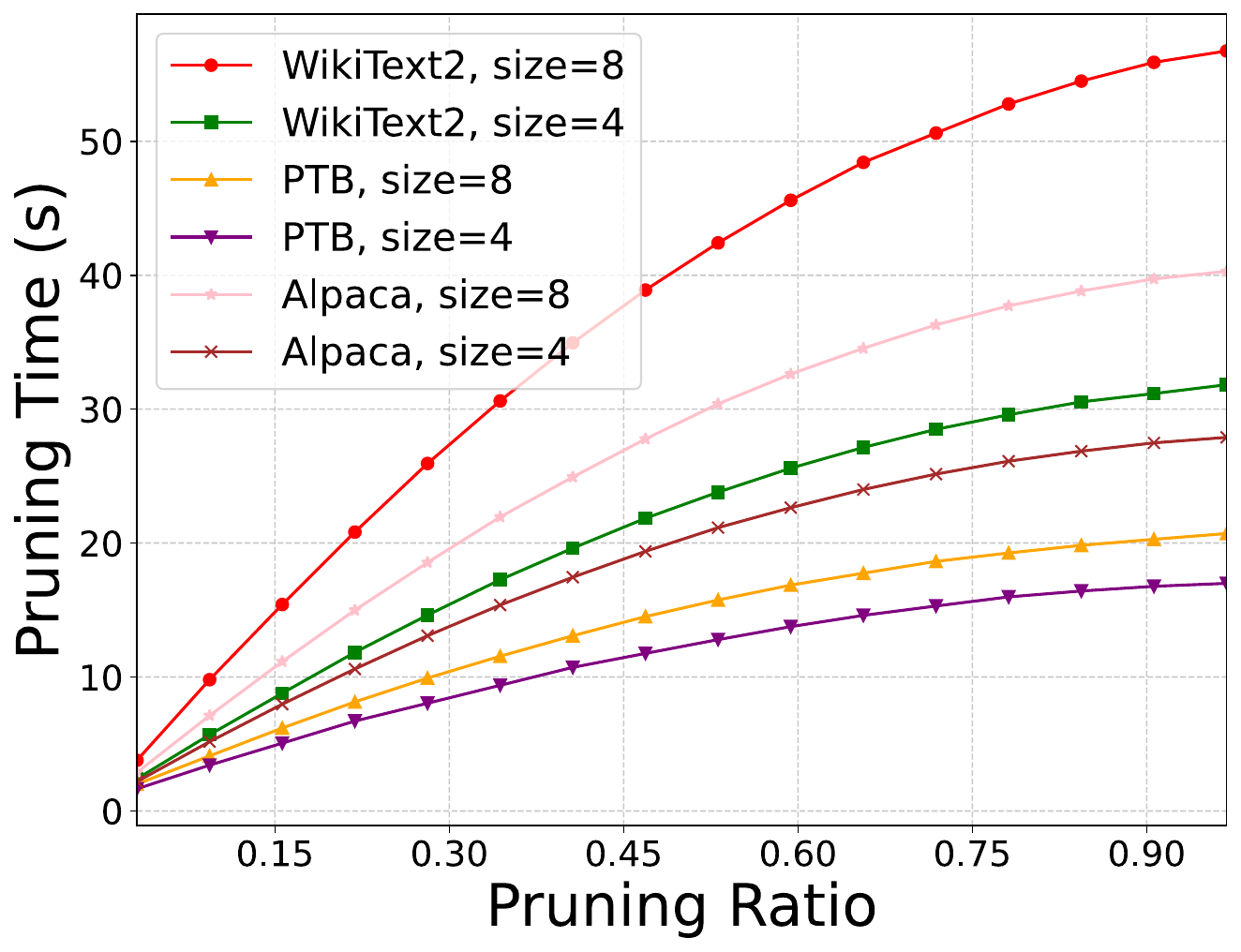}
    \caption{Impact of pruning ratio on pruning time.}
    \label{fig:pruning_time}
\end{figure}
Let the average time for forward propagation in a layer be $t_f$, and backward propagation be $t_b$. Using the PIP method, the time required to prune $m$ layers of model $\mathcal M$ is denoted as $PT_{\mathcal M} (m)$. The second finite difference of pruning time is calculated as:
\begin{align}
\triangle ^{2} PT_{\mathcal M} (m)=-(t_f+t_b),    
\end{align}
where $\triangle ^{2}$ denotes the second finite difference. This result implies constant second differences in $PT_{\mathcal M} (m)$, a hallmark of quadratic sequences. For a detailed proof, please see Appendix \ref{appendix:proof_second}. This is consistent with Figure \ref{fig:pruning_time}, which shows quadratic scaling of PIP pruning time for the LLaMA3-8B model across datasets and sample sizes.

\subsubsection{Effect of Pruning Ratio}

As shown in Table \ref{tab:ratio}, increasing the pruning ratio from 29.2\% to 39.0\% reduces memory use and latency, while causing only modest increases in perplexity and minor decreases in accuracy. Beyond 39.0\%, further pruning yields diminishing resource savings but leads to exponential increases in perplexity and substantial drops in accuracy. These results demonstrate that a pruning ratio of approximately 39.0\% achieves the optimal balance between computational efficiency and model quality.

\begin{table}[h]
    \centering
    \begin{tabular}{r|rr|r|r}
    \toprule 
    \toprule 
    \textbf{Ratio}  &\textbf{Memory} &\textbf{TPOT} &\textbf{PPL↓} &\textbf{Acc↑} \\ 
    \midrule
    29.2\% & 17.2GiB & 46.8ms & 53.3 & 48.9 \\
    \midrule
    39.0\% & 14.9GiB & 41.2ms & 95.8 & 44.7 \\
    \midrule
    48.7\% & 12.5GiB & 35.3ms & 379.0 & 38.5 \\
    \midrule
    58.5\% & 10.1GiB & 29.1ms & 1030.2 & 34.2 \\
    \midrule
    68.2\% & 6.5GiB & 20.1ms & 2549.7 & 33.1 \\
    \midrule
    78.0\% & 4.2GiB & 14.5ms & 22026.5 & 34.1 \\
    \bottomrule
    \bottomrule
    \end{tabular}
    \caption{Performance of PIP-pruned LLaMA2-13B. ``Ratio''  is the pruning ratio, and ``Acc'' is the average accuracy of the pruned model on various benchmarks.}
    \label{tab:ratio}
\end{table}

\subsubsection{Orthogonal to Quantization}

\begin{table}[h]
    \centering
    \begin{tabular}{r|rr|r|r}
    \toprule 
    \toprule 
    \textbf{Ratio}  &\textbf{Memory} &\textbf{TPOT} &\textbf{PPL↓} &\textbf{Acc↑} \\ 
    \midrule
    9.8\% & 11.4GiB & 181.8ms & 30.4 & 57.8 \\
    \midrule
    19.5\% & 10.2GiB & 160.9ms & 42.6 & 53.0 \\
    \midrule
    29.2\% & 9.0GiB & 120.8ms & 68.7 & 48.6 \\
    \midrule
    39.0\% & 7.7GiB & 102.1ms & 114.7 & 44.0 \\
    \bottomrule
    \bottomrule
    \end{tabular}
    \caption{Performance of quantized and PIP-pruned LLaMA2-13B. ``Ratio''  refers to the pruning ratio, and ``Acc'' represents the average accuracy of the quantized and pruned model on various benchmarks.}
    \label{tab:Quantization}
\end{table}

PIP is orthogonal to other compression approaches. As shown in Table \ref{tab:Quantization}, combining 8-bit quantization with PIP pruning achieves substantial memory reduction and speed-ups while maintaining acceptable perplexity and accuracy. This confirms that quantization and PIP pruning are fully orthogonal, enabling effective stacked compression.

\section{Conclusion and Future Work}

In this paper, we propose PIP, a novel perturbation-based iterative structured pruning method that unifies unperturbed and perturbed model views. It demonstrates theoretical rigor and SOTA performance across multiple benchmarks.

Future work could explore adaptive perturbation mechanisms, like dynamic scaling or task-specific perturbations, to enhance the precision of pruning. Additionally, we aim to collaborate with industry partners to deploy PIP in real-world applications, such as edge computing systems, to validate its practicality and address potential issues.

\section*{Limitations}

A key limitation of PIP is its current incompatibility with multimodal models. Tailored for text, PIP's perturbation and gradient analysis aren't easily adaptable to diverse data types like images or audio. This poses a challenge due to the unique processing needs of different data forms. Consequently, to broaden PIP's applicability, it's crucial to develop methods that handle multimodal data complexities. Overcoming this limitation will enhance PIP's utility in various AI applications.

\section*{Acknowledgements}
This work is supported by the National Natural Science Foundation of China (Grant No. 62272334, 6257073827).

\bibliography{custom}

\appendix

\begin{appendices}
\section{Definition of Norm}\label{appendix:norm}
In this section, we provide the formal definitions of the norms used in the main text, specifically the $L_1$, $L_2$, and $L_\infty$ norms. These norms are commonly used in various mathematical and computational contexts to measure the magnitude of vectors.
\subsection{$L_1$-norm}
For a vector $\mathbf{x} = (x_1, x_2, \ldots, x_n) \in \mathbb{R}^n$, the $L_1$-norm (also known as the Manhattan norm or Taxicab norm) is defined as:
\begin{align}
\|\mathbf{x}\|_1 = \sum_{i=1}^n |x_i|.
\end{align}
This norm represents the sum of the absolute values of the vector components.

\subsection{$L_2$-norm}
The $L_2$-norm (also known as the Euclidean norm) is perhaps the most commonly used norm. For a vector $\mathbf{x} = (x_1, x_2, \ldots, x_n) \in \mathbb{R}^n$, it is defined as:
\begin{align}
\|\mathbf{x}\|_2 = \sqrt{\sum_{i=1}^n x_i^2}.
\end{align}
This norm represents the Euclidean length of the vector, which is the geometric distance from the origin to the point represented by the vector in $n$-dimensional space.

\subsection{$L_\infty$-norm}
The $L_\infty$-norm (also known as the maximum norm or infinity norm) is defined as:
\begin{align}
\|\mathbf{x}\|_\infty = \max_{i} |x_i|.
\end{align}
This norm represents the maximum absolute value among the vector components. 

\section{Proof of Constant Second Differences}
\label{appendix:proof_second}
We prove that if the second finite difference of pruning time satisfies $\triangle^2 PT_{\mathcal{M}}(m) = -(t_f + t_b)$ for all $m$, then $PT_{\mathcal{M}}(m)$ must be a quadratic sequence. 

\begin{proof}
Let $\triangle^2 PT_{\mathcal{M}}(m) \equiv -(t_f + t_b)$ be constant. By definition, the second difference is the difference of consecutive first differences:  
\begin{align}
     \triangle PT_{\mathcal{M}}(m+1) - \triangle PT_{\mathcal{M}}(m) = -(t_f + t_b).
\end{align}
This implies that the first differences $\triangle PT_{\mathcal{M}}(m)$ form an arithmetic sequence with common difference $-(t_f + t_b)$. Explicitly, for the initial first difference $\triangle PT_{\mathcal{M}}(0)$, we have:
\begin{align}
    \triangle PT_{\mathcal{M}}(m) &= \triangle PT_{\mathcal{M}}(0) - (t_f + t_b)m.
\end{align}

The pruning time $PT_{\mathcal{M}}(m)$ is then the cumulative sum of these first differences:
\begin{align}
    & PT_{\mathcal{M}}(m) \nonumber 
    \\&= PT_{\mathcal{M}}(0) + \sum_{k=0}^{m-1} \triangle PT_{\mathcal{M}}(k) \nonumber \\
    &= PT_{\mathcal{M}}(0) + \sum_{k=0}^{m-1} \left[\triangle PT_{\mathcal{M}}(0) - (t_f + t_b)k \right] \nonumber \\
    &= PT_{\mathcal{M}}(0) + m \triangle PT_{\mathcal{M}}(0) \nonumber \\ 
     & \hspace{1em} - (t_f + t_b) \frac{m(m-1)}{2}.
\end{align}

Letting $a \equiv - \frac{t_f + t_b}{2}$, $b \equiv \triangle PT_{\mathcal{M}}(0) + \frac{t_f + t_b}{2}$, and $c \equiv PT_{\mathcal{M}}(0)$, this simplifies to:
\begin{align}
    PT_{\mathcal{M}}(m) &= a m^2 + b m + c,
\end{align}
which is explicitly a quadratic function of $m$.
\end{proof}

\begin{table}[tb]
\centering
\setlength{\tabcolsep}{0.1mm}{
\begin{tabular}{c|c|l}
\toprule
\toprule
\textbf{Model} & \textbf{Ratio} & \multicolumn{1}{c}{\textbf{Pruned Layers}}\\
\midrule
LLaMA3-8B & 19.0\% & 22, 18, 23, 28, 19, 21, 27 \\
\midrule
\multirow{2}{*}{LLaMA3-70B} & \multirow{2}{*}{19.4\%} & 06, 11, 46, 50, 75, 34, 49, 04, \\
 & & 20, 25, 48, 57, 60, 56, 55, 58 \\
\midrule
LLaMA2-13B & 19.5\% & 36, 31, 28, 13,
 35, 25, 38, 23 \\
\midrule
\multirow{2}{*}{LLaMA2-70B} & \multirow{2}{*}{19.9\%} & 59, 29, 67, 26, 60, 61, 50, 43, \\
 & & 58, 57, 31, 17, 56, 74, 62, 49 \\
\bottomrule
\bottomrule
\end{tabular}
}
\caption{Pruned models by PIP in our main experiments.}
\label{tab:main_prunedmodel}
\end{table}

\section{The Main Experiment}\label{appendix:main}

Our main experiments validate PIP across four model scales (8B-70B parameters), with pruned models in Table~\ref{tab:main_prunedmodel}. All configurations use fewer than 10 calibration samples from WikiText2 and employ the $L_2$-norm gradient aggregation strategy.

\section{More Analysis}
\subsection{Generalization to Other Series’ LLMs}\label{appendix:other_llm}
We aim to demonstrate the generalization capability of our PIP pruning method beyond the LLaMA models. We have conducted extensive experiments on the OPT series of models with varying scales (ranging from 1.3B to 13B parameters) to validate the effectiveness and robustness of our approach. The results are summarized in Table~\ref{tab:opt_results}.

\subsection{Effect of Gradient Aggregation}
\label{appendix:grad_aggr}
We systematically evaluate gradient aggregation norms ($L_1$, $L_2$, $L_\infty$) across model architectures under controlled pruning settings. The LLaMA3-8B experiments employ 8 calibration samples with character swap perturbation and 7-layer pruning, while LLaMA2-13B utilizes 4 samples with character replacement perturbation and 8-layer pruning. Table~\ref{tab:grad_aggr} presents the zero-shot benchmark results under various gradient aggregation strategies.

\subsection{Effect of Calibration Data Volume}
\label{appendix:data_volume}
For LLaMA3-8B, experiments use WikiText2 as the calibration dataset. Layer importance scores are computed through $L_2$-norm aggregation of gradients, computing layer importance via $L_2$-norm aggregation of gradients under character swap perturbation. The LLaMA2-13B configuration maintains the $L_2$-norm aggregation and dataset while employing character-replacement perturbation. The complete lists of pruned layers (0-indexed) under each condition are provided in Table~\ref{tab:data_volume_llamapruned}.

\subsection{Effect of Pruning Ratio}
We systematically evaluate the performance degradation of LLaMA3 and LLaMA2 models under increasing pruning ratios (10\%–30\%), as detailed in Table~\ref{tab:pruning_ratio}. As pruning ratios rise (10\%→30\%), both LLaMA3 and LLaMA2 show sharp performance drops. Commonsense tasks (ARC-Challenge, OBQA) decline the most, while language tasks (BoolQ, WinoGrande) are more robust.

\subsection{Effect of Text Perturbation Method}
The experiments are conducted on LLaMA3-8B under three text perturbation methods (\textit{Swap}, \textit{Replace}, \textit{Insert}) with fixed configurations: 4 calibration samples, $L_2$-norm gradient aggregation. Table~\ref{tab:tpm} compares the zero-shot performance across various benchmarks under different text perturbation methods. All perturbation methods degrade LLaMA3's zero-shot performance, with \textit{Replace} showing the least decline.

\subsection{Effect of Calibration Dataset}
The experiments evaluate the LLaMA3-8B model under three calibration datasets (WikiText2, PTB, Alpaca) with fixed configurations: 8 calibration samples, swap-based text perturbation, $L_2$-norm gradient aggregation. Table \ref{tab:cd} shows that PIQA and WinoGrande are stable across datasets, while ARC-Challenge declines sharply on PTB.

\subsection{Short Generations from Pruned Models}

The examples in Table \ref{tab:generation} clearly demonstrate that the pruned LLMs, despite undergoing the PIP pruning method, retain robust language expression capabilities. For instance, the pruned LLaMA3 model is capable of generating a coherent and insightful statement about the impact of AI on the business world, highlighting its potential to change the future of work. Similarly, the pruned LLaMA2 model provides a comprehensive introduction to NLP, emphasizing its role in processing human languages and extracting valuable insights from unstructured text data. These examples validate our pruning methodology's effectiveness in preserving core linguistic competencies, particularly in: technical semantic preservation, syntactic coherence across multi-clause constructions, and domain-appropriate register consistency.

\subsection{Long Generations from Pruned Models}
\label{appendix:gen}
Tables \ref{tab:gen1} and \ref{tab:gen2} demonstrate that models pruned via the PIP method retain text generation quality comparable to their dense counterparts. Pruned models preserve logical coherence and domain-specific knowledge (e.g., technical terminology and contextual reasoning), with minimal degradation in fluency and factual accuracy, validating PIP's capability to identify and retain critical layers.

\begin{table*}[htbp]
\centering
\setlength{\tabcolsep}{7mm}{
\begin{tabular}{cc|ccc}
\toprule \toprule 
\textbf{Model} & \textbf{Ratio} & \textbf{Memory} & \textbf{TPOT} & \textbf{PPL↓} \\ 
\midrule
OPT-1.3B & \textit{Dense} & 2.4508GiB & 29.5556ms & 16.9137 \\
  \cline{2-5}
  & 12.5\% & 2.1694GiB & 27.5401ms & 20.4018 \\
  \cline{2-5}
  & 20.8\% & 1.9818GiB & 26.0304ms & 38.7159 \\
  \cline{2-5}
  & 29.2\% & 1.7942GiB & 24.5175ms & 98.8645 \\
\midrule
OPT-2.7B & \textit{Dense} & 4.9395GiB & 41.3229ms & 15.1614 \\
  \cline{2-5}
  & 12.5\% & 4.3533GiB & 40.8443ms & 19.4676 \\
  \cline{2-5}
  & 21.9\% & 3.9137GiB & 30.1659ms & 24.9969 \\
  \cline{2-5}
  & 31.2\% & 3.4740GiB & 23.1771ms & 44.5617 \\
\midrule
OPT-6.7B & \textit{Dense} & 12.4024GiB & 45.1802ms & 13.1724 \\
  \cline{2-5}
  & 12.5\% & 10.9020GiB & 35.8956ms & 15.6426 \\
  \cline{2-5}
  & 21.9\% & 9.7767GiB & 33.8391ms & 19.7741 \\
  \cline{2-5}
  & 31.2\% & 8.6514GiB & 31.9243ms & 35.2512 \\
\midrule
OPT-13B & \textit{Dense} & 23.9415GiB & 71.5918ms & 12.3743 \\
  \cline{2-5}
  & 10.0\% & 21.5972GiB & 57.4137ms & 14.0220 \\
  \cline{2-5}
  & 20.0\% & 19.2530GiB & 50.1614ms & 17.7254 \\
  \cline{2-5}
  & 30.0\% & 16.9087GiB & 46.7974ms & 23.8522 \\
\bottomrule \bottomrule
\end{tabular}
}
\caption{Pruning Results on OPT Models.}
\label{tab:opt_results}
\end{table*}

\begin{table*}[t]
\centering
\setlength{\tabcolsep}{0.5mm}{
\begin{tabular}{cc|l|rrrrrrr|r}
\toprule
\toprule
\textbf{Model} & \textbf{Norm} & \multicolumn{1}{c|}{\textbf{Pruned Layers}} & \textbf{BoolQ↑} & \textbf{PIQA↑} & \textbf{HeSwg↑} & \textbf{WGrd↑} & \textbf{ARC-E↑} & \textbf{ARC-C↑} & \textbf{OBQA↑} & \textbf{Avg.↑}\\
\midrule
LLaMA3 & \textit{Dense} & -- & 81.4 & 79.7 & 60.2 & 72.5 & 80.1 & 50.5 & 34.8 & 65.6 \\ \cline{2-11}
& \multirow{2}{*}{$L_1$} & 23, 18, 29, 22, & \multirow{2}{*}{49.8} & \multirow{2}{*}{54.1} & \multirow{2}{*}{33.7} & \multirow{2}{*}{57.1} & \multirow{2}{*}{41.5} & \multirow{2}{*}{27.0} & \multirow{2}{*}{27.4} & \multirow{2}{*}{41.5} \\
& & 11, 17, 10 & & & & & & & & \\ \cline{2-11}
& \multirow{2}{*}{$L_2$} & 23, 22, 18, 19, & \multirow{2}{*}{52.3} & \multirow{2}{*}{53.8} & \multirow{2}{*}{35.5} & \multirow{2}{*}{59.4} & \multirow{2}{*}{41.0} & \multirow{2}{*}{30.9} & \multirow{2}{*}{29.0} & \multirow{2}{*}{43.1} \\
& & 28, 20, 31 & & & & & & & & \\ \cline{2-11}
& \multirow{2}{*}{$L_\infty$} & 20, 13, 18, 23, & \multirow{2}{*}{56.9} & \multirow{2}{*}{54.9} & \multirow{2}{*}{33.6} & \multirow{2}{*}{54.1} & \multirow{2}{*}{34.7} & \multirow{2}{*}{26.3} & \multirow{2}{*}{24.8} & \multirow{2}{*}{40.7} \\
& & 12, 28, 22 & & & & & & & & \\
\midrule
LLaMA2 & \textit{Dense} & -- & 80.6 & 79.1 & 60.0 & 72.4 & 79.4 & 48.5 & 35.2 & 65.0 \\ \cline{2-11}
& \multirow{2}{*}{$L_1$} & 00, 34, 35, 09, & \multirow{2}{*}{37.9} & \multirow{2}{*}{51.7} & \multirow{2}{*}{25.6} & \multirow{2}{*}{48.0} & \multirow{2}{*}{25.5} & \multirow{2}{*}{22.4} & \multirow{2}{*}{11.8} & \multirow{2}{*}{31.8} \\
& & 12, 27, 10, 33 & & & & & & & & \\ \cline{2-11}
& \multirow{2}{*}{$L_2$} & 36, 31, 28, 13, & \multirow{2}{*}{63.3} & \multirow{2}{*}{74.5} & \multirow{2}{*}{50.5} & \multirow{2}{*}{62.0} & \multirow{2}{*}{58.8} & \multirow{2}{*}{37.4} & \multirow{2}{*}{25.0} & \multirow{2}{*}{53.1} \\
& & 35, 25, 38, 23 & & & & & & & & \\ \cline{2-11}
& \multirow{2}{*}{$L_\infty$} & 30, 27, 24, 03, & \multirow{2}{*}{62.1} & \multirow{2}{*}{63.9} & \multirow{2}{*}{38.3} & \multirow{2}{*}{55.6} & \multirow{2}{*}{44.6} & \multirow{2}{*}{25.3} & \multirow{2}{*}{16.4} & \multirow{2}{*}{43.8} \\
& & 28, 29, 17, 13 & & & & & & & & \\
\bottomrule
\bottomrule
\end{tabular}
}
\caption{Zero-shot performance under different gradient aggregation strategies.}
\label{tab:grad_aggr}
\end{table*}

\begin{table*}[h]
\centering
\setlength{\tabcolsep}{0.5mm}{
\begin{tabular}{cc|l|rrrrrrr|r}
\toprule
\toprule
\textbf{Model} & \textbf{Cnt.}  & \multicolumn{1}{c|}{\textbf{Pruned Layers}} & \textbf{BoolQ↑} & \textbf{PIQA↑} & \textbf{HeSwg↑} & \textbf{WGrd↑} & \textbf{ARC-E↑} & \textbf{ARC-C↑} & \textbf{OBQA↑} & \textbf{Avg.↑}  \\
\midrule
LLaMA3 & \textit{Dense} & -- & 81.4 & 79.7 & 60.2 & 72.5 & 80.1 & 50.5 & 34.8 & 65.6 \\ \cline{2-11}
& \multirow{2}{*}{2}  & 23, 22, 18, 30, & \multirow{2}{*}{74.0} & \multirow{2}{*}{68.7} & \multirow{2}{*}{42.8} & \multirow{2}{*}{68.9} & \multirow{2}{*}{57.7} & \multirow{2}{*}{34.9} & \multirow{2}{*}{25.4} & \multirow{2}{*}{53.2} \\
& & 25, 27, 26 & & & & & & & & \\ \cline{2-11}
& \multirow{2}{*}{4}  & 31, 22, 24, 25, & \multirow{2}{*}{62.3} & \multirow{2}{*}{70.1} & \multirow{2}{*}{41.8} & \multirow{2}{*}{61.2} & \multirow{2}{*}{61.8} & \multirow{2}{*}{37.1} & \multirow{2}{*}{28.4} & \multirow{2}{*}{51.8} \\
& & 30, 28, 07 & & & & & & & & \\ \cline{2-11}
& \multirow{2}{*}{8}  & 23, 18, 22, 19, & \multirow{2}{*}{76.0} & \multirow{2}{*}{69.1} & \multirow{2}{*}{44.5} & \multirow{2}{*}{67.1} & \multirow{2}{*}{55.7} & \multirow{2}{*}{34.4} & \multirow{2}{*}{24.8} & \multirow{2}{*}{53.1} \\
& & 28, 21, 20 & & & & & & & & \\ \cline{2-11}
& \multirow{2}{*}{16} & 23, 31, 24, 21, & \multirow{2}{*}{70.5} & \multirow{2}{*}{68.4} & \multirow{2}{*}{45.5} & \multirow{2}{*}{63.6} & \multirow{2}{*}{56.7} & \multirow{2}{*}{38.3} & \multirow{2}{*}{27.4} & \multirow{2}{*}{52.9} \\
& & 25, 22, 18 & & & & & & & & \\ \cline{2-11}
& \multirow{2}{*}{32} & 31, 30, 22, 10, & \multirow{2}{*}{65.9} & \multirow{2}{*}{68.0} & \multirow{2}{*}{40.1} & \multirow{2}{*}{59.7} & \multirow{2}{*}{56.3} & \multirow{2}{*}{31.3} & \multirow{2}{*}{25.4} & \multirow{2}{*}{49.5} \\
& & 05, 17, 21 & & & & & & & & \\ \cline{2-11}
& \multirow{2}{*}{64} & 31, 25, 23, 22, & \multirow{2}{*}{66.2} & \multirow{2}{*}{70.2} & \multirow{2}{*}{47.8} & \multirow{2}{*}{65.4} & \multirow{2}{*}{58.5} & \multirow{2}{*}{39.5} & \multirow{2}{*}{28.8} & \multirow{2}{*}{53.8} \\
& & 28, 24, 19 & & & & & & & & \\
\midrule
LLaMA2 & \textit{Dense} & -- & 80.6 & 79.1 & 60.0 & 72.4 & 79.4 & 48.5 & 35.2 & 65.0 \\ \cline{2-11}
& \multirow{2}{*}{2}  & 09, 25, 34, 21, & \multirow{2}{*}{62.5} & \multirow{2}{*}{73.9} & \multirow{2}{*}{49.0} & \multirow{2}{*}{60.7} & \multirow{2}{*}{65.2} & \multirow{2}{*}{32.7} & \multirow{2}{*}{25.4} & \multirow{2}{*}{52.8} \\
& & 14, 19, 31, 06 & & & & & & & & \\ \cline{2-11}
& \multirow{2}{*}{4}  & 36, 31, 28, 13, & \multirow{2}{*}{63.3} & \multirow{2}{*}{74.5} & \multirow{2}{*}{50.5} & \multirow{2}{*}{62.0} & \multirow{2}{*}{58.8} & \multirow{2}{*}{37.4} & \multirow{2}{*}{25.0} & \multirow{2}{*}{53.1} \\
& & 35, 25, 38, 23 & & & & & & & & \\ \cline{2-11}
& \multirow{2}{*}{8}  & 34, 36, 31, 21, & \multirow{2}{*}{64.8} & \multirow{2}{*}{73.0} & \multirow{2}{*}{50.3} & \multirow{2}{*}{64.1} & \multirow{2}{*}{60.4} & \multirow{2}{*}{33.8} & \multirow{2}{*}{24.2} & \multirow{2}{*}{52.9} \\
& & 26, 22, 07, 05 & & & & & & & & \\ \cline{2-11}
& \multirow{2}{*}{16} & 29, 08, 27, 30, & \multirow{2}{*}{62.4} & \multirow{2}{*}{74.0} & \multirow{2}{*}{51.6} & \multirow{2}{*}{64.3} & \multirow{2}{*}{60.9} & \multirow{2}{*}{37.3} & \multirow{2}{*}{26.0} & \multirow{2}{*}{53.8} \\
& & 25, 35, 23, 17 & & & & & & & & \\ \cline{2-11}
& \multirow{2}{*}{32} & 31, 33, 17, 23, & \multirow{2}{*}{69.4} & \multirow{2}{*}{73.8} & \multirow{2}{*}{49.9} & \multirow{2}{*}{62.1} & \multirow{2}{*}{57.2} & \multirow{2}{*}{31.9} & \multirow{2}{*}{25.8} & \multirow{2}{*}{52.9} \\
& & 32, 19, 16, 14 & & & & & & & & \\ \cline{2-11}
& \multirow{2}{*}{64} & 36, 33, 17, 30, & \multirow{2}{*}{62.8} & \multirow{2}{*}{74.4} & \multirow{2}{*}{52.4} & \multirow{2}{*}{63.5} & \multirow{2}{*}{61.7} & \multirow{2}{*}{36.9} & \multirow{2}{*}{27.0} & \multirow{2}{*}{54.1} \\
& & 24, 27, 13, 31 & & & & & & & & \\
\bottomrule
\bottomrule
\end{tabular}}
\caption{Zero-shot performance across different calibration sample counts (Cnt.).}
\label{tab:data_volume_llamapruned}
\end{table*}

\begin{table*}[h]
\centering
\setlength{\tabcolsep}{0.5mm}{
\begin{tabular}{cc|l|rrrrrrr|r}
\toprule
\toprule
\textbf{Model} & \textbf{Ratio} & \multicolumn{1}{c|}{\textbf{Pruned Layers}} & \textbf{BoolQ↑} & \textbf{PIQA↑} & \textbf{HeSwg↑} & \textbf{WGrd↑} & \textbf{ARC-E↑} & \textbf{ARC-C↑} & \textbf{OBQA↑} & \textbf{Avg.↑} \\
\midrule
LLaMA3 
& \textit{Dense} & -- & 81.4 & 79.7 & 60.2 & 72.5 & 80.1 & 50.5 & 34.8 & 65.6 \\ \cline{2-11}
& 10.9\% & 22, 18, 23, 28 & 78.8 & 75.2 & 53.2 & 72.8 & 71.0 & 43.3 & 29.4 & 60.5 \\ \cline{2-11}
& \multirow{2}{*}{19.0\%} & 22, 18, 23, 28 & \multirow{2}{*}{70.9} & \multirow{2}{*}{69.6} & \multirow{2}{*}{44.7} & \multirow{2}{*}{69.4} & \multirow{2}{*}{57.9} & \multirow{2}{*}{35.1} & \multirow{2}{*}{26.8} & \multirow{2}{*}{53.5} \\
& & 19, 21, 27 &&&&&&& \\ \cline{2-11}
& \multirow{3}{*}{29.9\%} & 22, 18, 23, 28 & \multirow{3}{*}{43.5} & \multirow{3}{*}{63.0} & \multirow{3}{*}{36.4} & \multirow{3}{*}{56.2} & \multirow{3}{*}{43.3} & \multirow{3}{*}{30.0} & \multirow{3}{*}{25.0} & \multirow{3}{*}{42.5} \\
& & 19, 21, 27, 10 &&&&&&& \\
& & 25, 06, 31 &&&&&&& \\
\midrule
LLaMA2
& \textit{Dense} & -- & 80.6 & 79.1 & 60.0 & 72.4 & 79.4 & 48.5 & 35.2 & 65.0 \\ \cline{2-11}
& 9.8\% & 36, 31, 28, 13 & 63.0 & 76.1 & 56.0 & 66.1 & 67.4 & 41.6 & 30.4 & 57.2 \\ \cline{2-11}
& \multirow{2}{*}{19.5\%} & 36, 31, 28, 13 & \multirow{2}{*}{63.3} & \multirow{2}{*}{74.5} & \multirow{2}{*}{50.5} & \multirow{2}{*}{62.0} & \multirow{2}{*}{58.8} & \multirow{2}{*}{37.4} & \multirow{2}{*}{25.0} & \multirow{2}{*}{53.1} \\
& & 35, 25, 38, 23 &&&&&&& \\ \cline{2-11}
& \multirow{3}{*}{29.2\%} & 36, 31, 28, 13 & \multirow{3}{*}{62.4} & \multirow{3}{*}{71.3} & \multirow{3}{*}{45.9} & \multirow{3}{*}{58.4} & \multirow{3}{*}{46.9} & \multirow{3}{*}{33.9} & \multirow{3}{*}{23.8} & \multirow{3}{*}{48.9} \\
& & 35, 25, 38, 23 &&&&&&& \\
& & 17, 26, 29, 30 &&&&&&& \\
\bottomrule
\bottomrule
\end{tabular}
}
\caption{Zero-shot performance comparison across different pruning ratios. ``Ratio'' refers to the pruning ratio.}
\label{tab:pruning_ratio}
\end{table*}

\begin{table*}[h]
\centering
\setlength{\tabcolsep}{0.4mm}{
\begin{tabular}{cc|l|rrrrrrr|r}
\toprule
\toprule
\textbf{Model} & \textbf{TPM} & \multicolumn{1}{c|}{\textbf{Pruned Layers}} & \textbf{BoolQ↑} & \textbf{PIQA↑} & \textbf{HeSwg↑} & \textbf{WGrd↑} & \textbf{ARC-E↑} & \textbf{ARC-C↑} & \textbf{OBQA↑} & \textbf{Avg.↑} \\
\midrule
LLaMA3 & \textit{Dense} & -- & 81.4 & 79.7 & 60.2 & 72.5 & 80.1 & 50.5 & 34.8 & 65.6 \\ \cline{2-11}
& \multirow{2}{*}{Swap} & 23, 22, 31, 25, & \multirow{2}{*}{63.9} & \multirow{2}{*}{69.5} & \multirow{2}{*}{44.7} & \multirow{2}{*}{63.6} & \multirow{2}{*}{58.7} & \multirow{2}{*}{36.0} & \multirow{2}{*}{31.2} & \multirow{2}{*}{52.5} \\
& & 16, 26, 30 &&&&&&&& \\ \cline{2-11}
& \multirow{2}{*}{Replace} & 22, 18, 23, 28, & \multirow{2}{*}{70.9} & \multirow{2}{*}{69.7} & \multirow{2}{*}{44.8} & \multirow{2}{*}{69.6} & \multirow{2}{*}{58.0} & \multirow{2}{*}{35.1} & \multirow{2}{*}{27.4} & \multirow{2}{*}{53.6} \\
& & 19, 21, 27 &&&&&&&& \\ \cline{2-11}
& \multirow{2}{*}{Insert} & 23, 18, 31, 28, & \multirow{2}{*}{71.3} & \multirow{2}{*}{70.2} & \multirow{2}{*}{47.6} & \multirow{2}{*}{64.4} & \multirow{2}{*}{57.8} & \multirow{2}{*}{36.3} & \multirow{2}{*}{25.0} & \multirow{2}{*}{53.2} \\
& & 22, 27, 03 &&&&&&&& \\
\bottomrule
\bottomrule
\end{tabular}
}
\caption{Zero-shot performance under different text perturbation methods. ``TPM'': Text Perturbation Method.}
\label{tab:tpm}
\end{table*}

\begin{table*}[h]
\centering
\setlength{\tabcolsep}{0.2mm}{
\begin{tabular}{cc|l|rrrrrrr|r}
\toprule
\toprule
\textbf{Model} & \textbf{CD} & \multicolumn{1}{c|}{\textbf{Pruned Layers}} & \textbf{BoolQ↑} & \textbf{PIQA↑} & \textbf{HeSwg↑} & \textbf{WGrd↑} & \textbf{ARC-E↑} & \textbf{ARC-C↑} & \textbf{OBQA↑} & \textbf{Avg.↑} \\
\midrule
LLaMA3 & \textit{Dense} & -- & 81.4 & 79.7 & 60.2 & 72.5 & 80.1 & 50.5 & 34.8 & 65.6 \\ \cline{2-11}
& \multirow{2}{*}{WikiText2} & 31, 28, 26, 29, & \multirow{2}{*}{80.6} & \multirow{2}{*}{69.3} & \multirow{2}{*}{42.8} & \multirow{2}{*}{62.3} & \multirow{2}{*}{61.2} & \multirow{2}{*}{39.2} & \multirow{2}{*}{30.6} & \multirow{2}{*}{55.1} \\
& & 23, 25, 30 &&&&&&&& \\ \cline{2-11}
& \multirow{2}{*}{PTB} & 23, 18, 28, 22, & \multirow{2}{*}{72.1} & \multirow{2}{*}{70.6} & \multirow{2}{*}{44.6} & \multirow{2}{*}{66.6} & \multirow{2}{*}{54.8} & \multirow{2}{*}{32.9} & \multirow{2}{*}{24.0} & \multirow{2}{*}{52.2} \\
& & 19, 21, 13 &&&&&&&& \\ \cline{2-11}
& \multirow{2}{*}{Alpaca} & 31, 30, 29, 28, & \multirow{2}{*}{80.5} & \multirow{2}{*}{70.7} & \multirow{2}{*}{42.7} & \multirow{2}{*}{63.1} & \multirow{2}{*}{60.5} & \multirow{2}{*}{38.6} & \multirow{2}{*}{30.4} & \multirow{2}{*}{55.2} \\
& & 23, 25, 27 &&&&&&&& \\
\bottomrule
\bottomrule
\end{tabular}
}
\caption{Zero-shot performance comparison across different calibration datasets. ``CD'': Calibration Dataset.}
\label{tab:cd}
\end{table*}

\begin{table*}[tb]
\hspace{-0.7em}
\setlength{\tabcolsep}{0.1mm}{
\begin{tabular}{c|c|p{14.3cm}}
\toprule 
\toprule 
\textbf{Model} & \textbf{✂} & \multicolumn{1}{c}{\textbf{Example}}    \\ 
\midrule
\multirow{4}{*}{LLaMA3}  & \multirow{2}{*}{\usym{2718}} & \underline{AI will} be widely used in all areas of business, says Bjarne Corydon. The minister of business and growth, Bjarne Corydon is excited about what AI is doing to the Danish business world.    \\ \cline{2-3}
&\multirow{2}{*}{\usym{2714}} &  \underline{AI will} change the world of work: this is what the Gartner report reveals. In the coming years, AI will have a large impact on the entire business, as well as the daily life of employees. \\
\midrule
\multirow{6}{*}{LLaMA2} & \multirow{3}{*}{\usym{2718}} &  \underline{NLP is} a way of applying computational processing to natural human languages that we use to communicate with each other. This course will give you an introduction to NLP, and how it can be leveraged to derive useful insights from unstructured text data.    \\ \cline{2-3}
 &\multirow{3}{*}{\usym{2714}} & \underline{NLP is} the science that focuses on linguistic data. It is an AI methodology that combines computer science and artificial intelligence. This science focuses on linguistic input, output, understanding, processing, or interaction. It is used to process human languages. \\
\bottomrule
\bottomrule
\end{tabular}
}
\caption{Examples from the original model and the pruned model. ``✂'' indicates whether it is pruned. The underline highlights the prompt provided to the LLM.}
\label{tab:generation}
\end{table*}

\begin{table*}[h]
\hspace{-0.7em}
\setlength{\tabcolsep}{0.1mm}{
\begin{tabular}{c|c|p{14.3cm}}
\toprule
\toprule
\textbf{Model} & \textbf{✂} & \multicolumn{1}{c}{\textbf{Example}} \\
\midrule
\multirow{46}{*}{LLaMA3}  & \multirow{22}{*}{\usym{2718}} & \underline{Large Language Models will} not replace us with AI – but they will help us automate our day-to-day work. AI has always been portrayed as the main reason behind replacing human labour with machine labour in many fields. For example, driverless cars and automated machines that replace employees are all examples of how AI will “replace” humans. However, AI has never replaced us, but has helped us automate many of our day-to-day operations and we will continue to see that trend of using AI in automation. One of the main reasons why AI has helped us improve our day-to-day work is that it is able to process, analyse and understand massive amounts of data. Large Language Models (LLMs) are a good example of this. LLMs can be a very useful tool to help us become more productive and improve our workflows. LLMs can also help us in many different ways, depending on what we are looking for, and they can be used in different ways. LLMs are able to process, analyse and understand massive amounts of text in many different ways. For example, they can be used to find information in a large amount of text, to do sentiment analysis, or to create new text. LLMs can also be used to do translation or to do a search in a large number of different languages. There are many tools available that allow us to work with LLMs, such as GPT-3 from OpenAI, which is one of the most powerful tools in the LLM space. LLMs can also be used to do sentiment analysis of text and we can use it in many different ways to help us become more productive and to improve our workflow. We can use LLMs to understand what is the sentiment of a large amount of text, to detect negative or positive sentiments in text, or to help us make decisions in our daily work. There are many ways to use LLMs to automate our day-to-day work. For example, we can use LLMs to help us become more productive and to improve our workflow by doing sentiment analysis. We can use LLMs to analyse large amounts of data and to analyse the sentiment of a large amount of text.   \\ \cline{2-3}
&\multirow{22}{*}{\usym{2714}} &  \underline{Large Language Models will} be the next killer apps. The way these LLMs are developed is not only more efficient but also very innovative in terms of how the LLMs are built, the use of compute to build the LLM, and how they train and generate LLMs. What’s more, we are seeing a change in how LLMs are used for tasks, with an increasing number of LLMs being used to train other AI models. This has opened up a huge market for LLMs, creating a great opportunity for companies like Cytora and Credence to tap into. In this article, we explore the benefits of this new wave of AI development, and how companies can capitalize on the opportunities presented by these new AI innovations. Large language models (LLMs) have been a focus of AI innovation for some time, and there are many ways they can be developed. However, the most effective and efficient way to develop LLMs is to use a lab environment that is specific to LLM development. There are two main approaches to developing LLMs: traditional computing methods, which require large amounts of computing power and data, and cloud-based LLMs, which can be developed much more quickly and inexpensively. There are several reasons why it is better to develop LLMs in an AI innovation lab than in a traditional development environment. First, the cloud-based LLM development process allows for much greater scale. With only a few machines in use, LLM developers can achieve huge breakthroughs, such as improving text generation accuracy. Second, cloud-based LLM development also allows for a much more nimble process. Instead of waiting for computers to complete tasks one at a time, developers can use a variety of machines to speed up their work. This speeds up the entire process, making it easier to get feedback from stakeholders and improving the overall experience of working with LLMs. How LLMs Can Be Used to Train Other AI Models. Large language models (LLMs) have been a driving force in the development of Artificial Intelligence (AI) over the years. \\
\bottomrule
\bottomrule
\end{tabular}
}
\caption{Examples generated by the original model and the pruned model. “✂” indicates whether it is pruned. The underline highlights the prompt provided to the LLM.}
\label{tab:gen1}
\end{table*}

\begin{table*}[h]
\hspace{-0.7em}
\setlength{\tabcolsep}{0.1mm}{
\begin{tabular}{c|c|p{14.3cm}}
\toprule
\toprule
\textbf{Model} & \textbf{✂} & \multicolumn{1}{c}{\textbf{Example}} \\
\midrule
\multirow{44}{*}{LLaMA2}  & \multirow{22}{*}{\usym{2718}} & \underline{Large Language Models will} be ubiquitous in business and government within 5-7 years, predicts PwC by James Sanders in Artificial Intelligence on January 30, 2023, 11:41 AM PST. Large Language Models such as OpenAI\'s GPT-3 are becoming increasingly prevalent for a wide variety of business applications, the consulting firm said. Image: Bjorn Rune Lie, Getty Images. The adoption rate of large language models—AI models trained on massive amounts of natural language data—will increase rapidly as businesses look to improve customer engagement and operational efficiencies. According to a report by PwC, 65\% of senior business executives indicated that large language models are a top investment area, and 53\% said that this is the largest investment area for AI technology in the coming year. "Making AI more accessible, through advances such as large language models, is essential to the democratization of the technology, which could bring a range of business benefits to organizations," said John Garner, global AI leader at PwC, in a statement accompanying the report. SEE: The ethical dilemmas of AI (ZDNET/TechRepublic special feature) | Download the free PDF version (TechRepublic) PwC\'s 2023 Global Artificial Intelligence Survey found that the use of natural language processing technologies is rapidly growing in the workplace. The technology is being used by 36\% of respondents to "identify risks or anomalies in client engagement," and by 35\% of respondents to "increase the effectiveness and efficiency of operations." AI is also used to drive productivity: 25\% of respondents indicated that AI is used to "enable the creation of new product and service offerings." The growing ubiquity of large language models in the workplace is also a factor driving widespread awareness: 55\% of business executives indicated that large language models are "extremely important" to business success and operations, and 79\% of employees said that they know of AI, a slight increase from last year\'s survey.  \\ \cline{2-3}
&\multirow{23}{*}{\usym{2714}} &  \underline{Large Language Models will} be used for everything from translation to financial services to healthcare. There are endless benefits to utilizing LLMs like ChatGPT, like saving money and time on repetitive tasks that are time-consuming or impossible to automate, and getting better answers than we could on our own. As AI gets more accessible to average users, a more accessible education in AI is more important than ever. The ChatGPT revolution has arrived. If you are a regular user of Google search or Twitter, you’ve probably already noticed. ChatGPT was released to the public in November 2022 by an organization called OpenAI. The platform uses artificial intelligence to create intelligent chatbot responses to user prompts. As a result, it has the potential to revolutionize the way we interact with technology. With ChatGPT, you can write essays, do your taxes, and ask questions like “Who wrote Romeo and Juliet?” or “Where is the nearest Walmart?” in chat format. It’s accessible, fast, and most importantly, free. It’s clear that LLMs are a powerful tool with enormous applications and capabilities. But what, exactly, is an LLM, and why is this technology so different from other language models that have been developed? What Are Large Language Models? An LLM is a type of language modeling that produces language using machine learning algorithms based on large amounts of training data. It’s a relatively new development in the field of artificial intelligence, and it has become increasingly popular in recent years due to the advances in natural language processing and understanding that have been made. One of the main reasons for this is that large language models are capable of producing language that is more sophisticated and accurate than ever before. There are a few key reasons why large language models are different from other language models: They are based on very large amounts of data: This is the key characteristic that sets LLMs apart from other language models. Because of the amount of data used, these models can be trained to perform more complex tasks and generate more human-like text. \\
\bottomrule
\bottomrule
\end{tabular}
}
\caption{Examples generated by the original model and the pruned model. “✂” indicates whether it is pruned. The underline highlights the prompt provided to the LLM.}
\label{tab:gen2}
\end{table*}
\end{appendices}
\end{document}